%% file: main.tex
\documentclass{article}

\usepackage{PRIMEarxiv}

\usepackage[utf8]{inputenc} 
\usepackage[T1]{fontenc}    
\usepackage{hyperref}       
\usepackage{url}            
\usepackage{booktabs}       
\usepackage{amsfonts}       
\usepackage{nicefrac}       
\usepackage{microtype}      
\usepackage{lipsum}
\usepackage{fancyhdr}       
\usepackage{graphicx}       
\graphicspath{{media/}}     

\usepackage{multirow}
\input{ide_preamble}

\title{Black-Box Anomaly Attribution
\thanks{This is an expanded version of \cite{Ide21AAAI}. Part of the content has also been presented in~\cite{ide2023generative}. The original version was submitted to a journal on May 8, 2021.
} 
}

\author{
  Tsuyoshi Id\'{e} \quad \quad  \quad
  Naoki~Abe \\
  IBM Research, Thomas J. Watson Research Center \\
1101 Kitchawan Rd., Yorktown Heights, NY 10598, USA\\
  \texttt{\{tide,nabe\}@us.ibm.com} 
}

\begin{document}
\maketitle

\begin{abstract}
    When the prediction of a black-box machine learning model deviates from the true observation, what can be said about the reason behind that deviation? This is a fundamental and ubiquitous question that the end user in a business or industrial AI application often asks. The deviation may be due to a sub-optimal black-box model, or it may be simply because the sample in question is an outlier. In either case, one would ideally wish to obtain some form of attribution score --- a value indicative of the extent to which an input variable is responsible for the anomaly. 
    In the present paper we address this task of ``anomaly attribution,'' particularly in the setting in which the model is black-box and the training data are not available. Specifically, we propose a novel likelihood-based attribution framework we call the ``likelihood compensation (LC),'' in which the responsibility score is equated with the correction on each input variable needed to attain the highest possible likelihood. We begin by showing formally why mainstream model-agnostic explanation methods, such as the local linear surrogate modeling and Shapley values, are not designed to explain anomalies. In particular, we show that they are ``deviation-agnostic,'' namely, that their explanations are blind to the fact that there is a deviation in the model prediction for the sample of interest. We do this by positioning these existing methods under the unified umbrella of a function family we call the ``integrated gradient family.'' We validate the effectiveness of the proposed LC approach using publicly available data sets. We also conduct a case study with a real-world building energy prediction task and confirm its usefulness in practice based on expert feedback.   
\end{abstract}



\section{Introduction}\label{sec:Introduction}
With the recent advances in machine learning algorithms and their wide-spread deployment, automated anomaly detection has become a critical component in many modern industrial systems. In its most ambitious form, it is coupled with the idea of a ``digital twin,'' which has recently emerged in the manufacturing industry~\cite{tao2018digital2,fuller2020digital}. The ``digital twin'' captures the behavior of an entire production system with a machine learning model.
As the model is a replica of the system under the normal operating conditions, any significant deviation from its prediction implies the presence of some anomalies. The expectation then is that critical incidents can be prevented by leveraging the digital twin model in a cycle of prediction, evaluation, and risk mitigation.

Despite the initial optimism, however, broad deployment of digital twins is still far from reality. One of the biggest concerns, as perceived by the end-user, is the lack of explainability, and hence actionability, of a typical digital twin model. Consider, for instance, the use-case of building energy management (see Sec.~\ref{sec:experiments} for more details). The energy consumption of a building, $y$, is predicted with a regression function $y=f(\bmx)$, where $\bmx$ is a vector of measurements such as the outside temperature and humidity. The monitoring system typically consists of a few sub-components including those for HVAC (heating, ventilating, and air conditioning), sensing and data collection, and data management. Since they are often developed by different vendors using proprietary technologies, the end-user may not have first-hand knowledge on the prediction model and the training data used to train it, even when they have the ownership of the overall system. In such a scenario, the prediction system acts as black-box, often leaving the end-user in the dark in cases of anomalies, i.e.\ prediction deviations. 

In this paper, we address the task of \textbf{anomaly attribution} in the black-box \textit{regression} setting. We assume the model is ``\textit{doubly black-box},'' meaning that we neither have access to the parametric form of the regression function $f(\bmx)$ nor its training data (see Table~\ref{table:black-box} for more details).  
In this especially challenging setting, we strive to provide a reliable ``attribution score'' for each of the input variables with regard to the anomaly in question. See Fig.~\ref{fig:problem_setting_and_deviation_v_increment} for illustration. We call this task ``anomaly attribution.'' A large deviation from the observation may be due to sub-optimal model training, or simply because the observed sample is an outlier. In either case, the attribution score indicative of the extent to which an input variable is responsible for the anomalous output will be useful in aiding the end-user make better decisions on what action to take, e.g.\ using their knowledge on the system.

\begin{figure}[t]
\begin{center}
\includegraphics[trim={0.5cm 4.cm 0.5cm 1cm},clip,width=12cm]{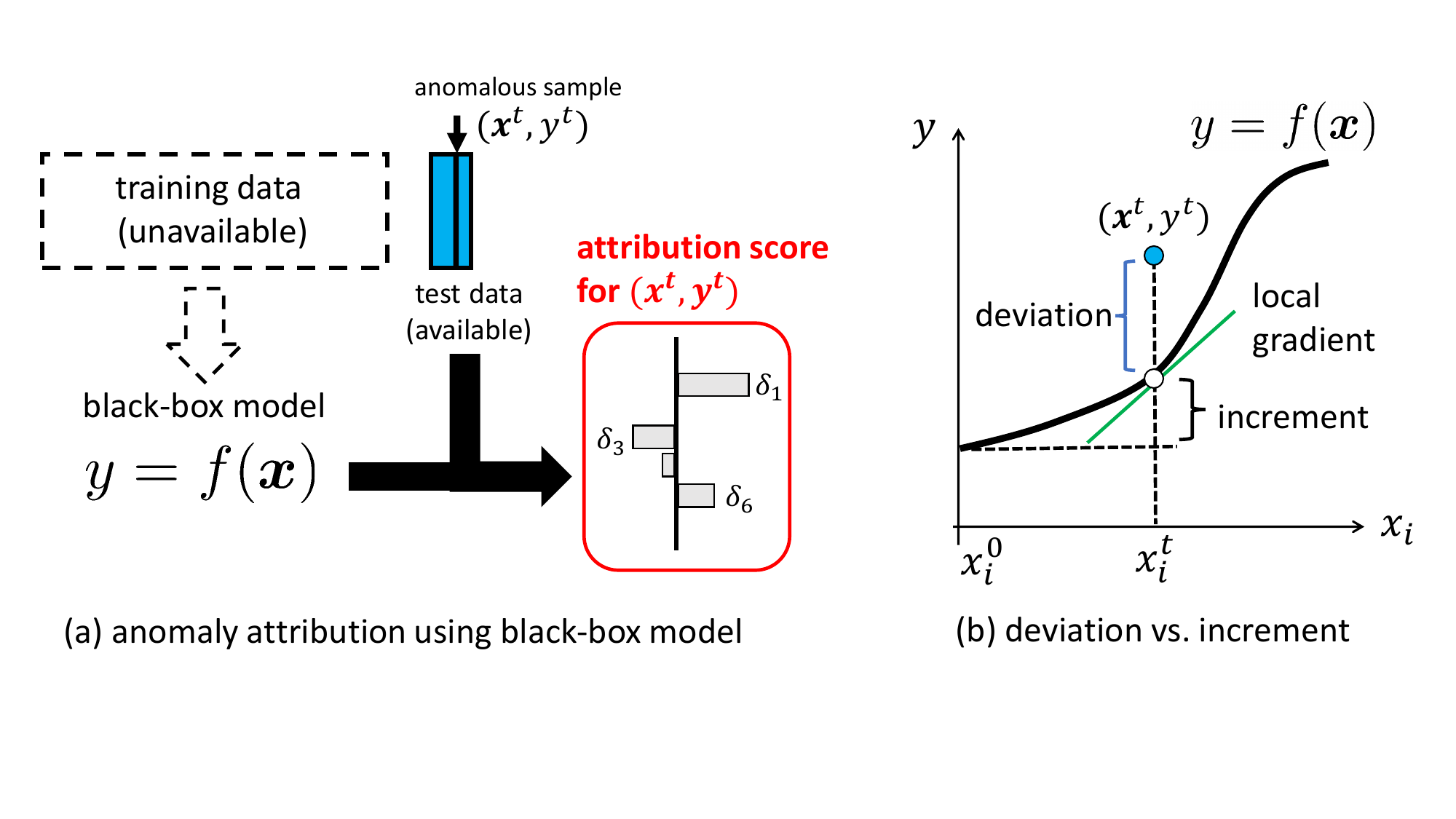}
\end{center}
\vspace{-0.2cm}
\caption{Problem setting and motivation. (a) Given a black-box regression model and anomalous sample(s), our goal is to quantify input variables' responsibility without using training data. (b)~Existing attribution methods attempt to explain either the local gradient or the increment from a reference point $\bmx^0$, rather than the deviation.}
\label{fig:problem_setting_and_deviation_v_increment}
\end{figure}

Anomaly attribution has previously been studied as a sub-task of anomaly detection. For instance, in subspace-based anomaly detection, computing each variable's responsibility has been part of the standard procedure in the white-box setting~\cite{Chandola09AnomalySurvey,Jiang11LKDD,dang2013outlier,dang2013local,micenkova2013explaining}. On the other hand, in the field of explainable artificial intelligence (XAI), growing attention has been paid on ``post-hoc'' explanations of black-box prediction models. Among numerous XAI techniques, as conveniently summarized in recent review articles~\cite{burkart2021survey,molnar2019interpretable,samek2019explainable,arrieta2020explainable,speith2022review}, there are a few methods that address model-agnostic feature attribution in the \textit{regression} setting: 
1) Local linear surrogate modeling, which is best known under the name LIME (Local Interpretable Model-agnostic Explanations)~\cite{ribeiro2016should}, 
2) Shapley value (SV), which was first introduced to the machine learning community by~\cite{kononenko2010efficient}, and
3) integrated gradient (IG)~\cite{sundararajan2017axiomatic}. 
We collectively call these approaches the \textit{function-based} approach as they work with the prediction function $f(\bmx)$ directly for attribution. One obvious limitation of these methods is that they are designed to explain what $f(\bmx)$ looks like at a test sample $\bmx^t$, rather than to explain the deviation $f(\bm{x}^t)-y^t$.

In this paper, we propose a novel \textit{likelihood-based} framework to anomaly attribution in the doubly black-box regression setting. Likelihood is a legitimate starting point for explaining deviation as it is a canonical metric of non-anomalousness. We begin by pointing out that existing attribution methods, LIME, SV, and IG, are inherently \textit{deviation-agnostic}, and, in fact, are \textit{not} appropriate for anomaly explanation. Interestingly, one can show that they are derived from a unified framework we call the \textit{integrated gradient family}, and the deviation-agnostic property is a general characteristic of the IG family. Based on the solid understanding of the function-based approach, we propose the new notion of \textit{likelihood compensation} (LC), which seeks a local perturbation that achieves the highest possible likelihood. As explained later (see Fig.~\ref{fig:LC_convergence.pdf} for a preview), LC as ``deviation measured horizontally'' is a useful attribution score as it is interpreted as an action that might be taken to bring back the outlying sample to normalcy. To the best of our knowledge, this is the first principled framework for model-agnostic anomaly attribution in the regression setting. 

The notion of LC was first introduced in our conference paper~\cite{Ide21AAAI}. This paper expands it by re-positioning the entire approach in a unified framework. In doing so, we formally show, for the first time, a close relationship among the function-based attribution algorithms, including the equivalence between SV and EIG. We also conduct a systematic empirical study using new datasets on the deviation agnostic property, score variability issues, and the consistency among different attribution approaches. In the last subsection, we present a real-world use-case, where the theory of LC made a substantial difference.

\section{Related Work}

This section summarizes prior works in the context of anomaly attribution. 

\subsection{General background: Anomaly attribution in doubly black-box setting}

As mentioned earlier, anomaly attribution has been studied as a sub-task of anomaly detection in the machine learning community, typically in the white-box unsupervised setting. 
In the modern deployment of AI systems, however, a black-box situation often arises~\cite{li2022survey}. Table~\ref{table:black-box} compares white- and black-box settings, although the definition of a black-box model varies in the literature. In the context of XAI, the white- and grey-box cases are typically associated with deep neural networks (DNNs). Since the number of network parameters is extremely large, the model is often viewed as black-box even with full access to the internal network parameters. Saliency maps~\cite{Simonyan14ICLR_workshop,Selvaraju_2017_ICCV} and layer-wise relevance propagation~\cite{montavon2019layer} are well-known DNN-specific attribution methods that fall into these categories. 

The doubly black-box scenario occurs when the end-user has access \textit{only} to the model's API (application programming interface) or does not have full understanding of the algorithm used. The latter can occur even when the source code is available. Since the nature of dependency on the internal parameters is unknown, for any XAI approach to be applicable to this setting it must be \textit{model-agnostic}, making most of the DNN-specific methods inapplicable.

\begin{table}[tbh]
    \caption{Comparison of different XAI settings. Our focus is the doubly black-box case.}
    \label{table:black-box}
    \centering
    \begin{tabular}{c c c c }
    \hline\hline
    & model API access & model internal access & training data access\\
    \hline
    white-box & yes & yes  & yes \\
    grey-box & yes &  yes & no \\
    \textbf{doubly black-box} & yes&  no & no \\
    \hline
    \end{tabular}
\end{table}

\subsection{Local linear modeling, Shapley value (SV), and integrated gradient (IG)}




As discussed before, we focus on model-agnostic post-hoc anomaly attribution in the doubly black-box regression setting. Here, 1) local linear surrogate modeling, 2) Shapley value (SV), and 3) integrated gradient (IG) are the main existing approaches that are potentially applicable to our task. Let us quickly review recent works of these approaches.

Local linear modeling has been extensively used for attribution for decades, often under the name of sensitivity analysis~\cite{abhishek2022attribution}. Recent applications to anomaly explanation include \cite{giurgiu2019additive} and~\cite{zhang2019ace}. The former used LIME for attribution inspired by the kernel SHAP approach~\cite{lundberg2017unified}. The latter also used LIME for attribution with a new regularization approach. Many gradient-based attribution methods can be viewed as local linear modeling, although many of them are DNN-specific, assuming full access to the internal parameters. 

SV is one of the most popular attribution methods in the AI community, and the last few years have seen many attempts to apply SV to various industrial domains. For instance, \cite{hwang2021sfd} proposed to use SV for sensor fault diagnosis, \cite{mariadass2022extreme} used SV to explain unexpected observations in crop yield analysis, and \cite{antwarg2021explaining} used SV to explain unusual warranty claims.


IG~\cite{sundararajan2017axiomatic} is another generic input attribution approach potentially applicable to the black-box setting. The application of IG to anomaly explanation can be found in, e.g.,~\cite{SippleICML2020,sipple2022general}.

Table~\ref{table:comparison chart} compares the key properties of LC, the proposed algorithm, with LIME, SV, and IG along with two additional methods: The expected integrated gradient (EIG), which generalizes IG by taking the expectation with respect to the baseline input (see Sec.~\ref{subsec:IG_and_EIG} for the detail) and the $Z$-score, which quantifies the deviation of each input variable from its expected value, independently of $y$. Since the true data distribution is unknown in general, the expectation has to be computed as the empirical approximation on the training data. The same applies to SV, which make them inapplicable to the doubly black-box setting, as shown in the `training-data-free' column. IG does not need the training data, but it does need an extra piece of information of the baseline input (the `baseline-free' column). Except for those domains in which a universally accepted data pre-processing method has been established, it is generally hard to choose a valid baseline input. The dependence on the baseline input is considered a major factor that limits practical utility of IG in anomaly attribution.  

One fundamental issue with the existing approaches is that they are \textit{deviation-agnostic} (the `$y$-sensitive' column), which will be mathematically shown in Sec.~\ref{sec:EIG}. This means that LIME, SV, and (E)IG, in fact, do not explain the reasons for the sample of interest to be anomalous: As illustrated in Fig.~\ref{fig:problem_setting_and_deviation_v_increment}~(b), they explain either the local gradient or the increment of $f(\bmx)$, rather than what may have caused the deviation $f(\bmx) - y$ at a test point $(\bmx,y)=(\bmx^t,y^t)$. This situation remains unchanged even if we apply these methods to the modified function $F(\bmx,y) \triangleq f(\bmx) - y$, as discussed later. 

\begin{table}[thb]
    \centering
    \caption{Comparison of different anomaly attribution methods in the regression setting. }
    \begin{tabular}{c c c c c}
    \hline \hline
            & training-data-free & baseline-free & $y$-sensitive  & reference point \\
    \hline
     LIME  & yes    & yes    &no & infinitesimal vicinity \\
     SV   & no    & yes    &no & globally distributional\\ 
     IG    & yes    & no    &no & arbitrary\\
     EIG   & no    & yes    &no & globally distributional \\
     Z-score& no    & yes    &no & global mean of predictors \\
     \textbf{LC}     & \textbf{yes}    & \textbf{yes}   &\textbf{yes} & maximum likelihood point \\
    \hline
    \end{tabular}
    \label{table:comparison chart}
\end{table}

\subsection{Unified attribution framework}

One of our contributions is establishing a unified framework for many of the popular ``function-based'' attribution methods. Prior work along this line includes~\cite{deng2021unified}, which attempts to characterize IG using Taylor expansion and proposes to use the expectation to neutralize the need for a specific baseline input.  \cite{sundararajan2020many} is another important work aiming at building a unified attribution framework. The authors pointed out that there can be a few different definitions for SV, such as the baseline SV and expected SV, and discussed the relationship with IG in a qualitative manner. \cite{lundberg2017unified} reintroduced the SV-based attribution method originally proposed by \cite{kononenko2010efficient} and proposed a hybrid method that lies between SV and LIME. Also, \cite{kumar2020problems} analyzed SV's risk of producing misleading attribution because of the gap between conditional and marginal distributions, while \cite{zhou2022feature} conducted a systematic empirical study to compare different attribution methods including LIME and SV.

Inspired by these works, we go one step further in this paper: Using power expansion, we mathematically show certain equivalence properties of the function-based attribution methods, including the equivalence between SV and EIG (Theorem~\ref{th:SV=EIG}), which naturally lead to the notion of the IG family. We then point out their two fundamental limitations in anomaly attribution that have been hitherto unnoticed. One is the deviation-agnostic property, and the other is the explicit or implicit dependency on arbitrary baseline points (as summarized in the `reference point' column in Table~\ref{table:comparison chart}). An in-depth analysis of the latter brings us to the new notion of likelihood-based attribution proposed in this paper, as discussed in Sec.~\ref{subsec:seeking_reference_point}.

\section{Problem Setting}\label{sec:problem_setting}

As mentioned earlier, we focus on the task of anomaly attribution in the \textit{regression} setting rather than classification or unsupervised settings. Figure~\ref{fig:problem_setting_and_deviation_v_increment}~(a) summarizes the overall problem setting. Suppose we have a (deterministic) regression model $y=f(\bm{x})$ in the doubly black-box setting (see Table~\ref{table:black-box}). Neither the training data set $\calD_{\mathrm{train}}$ nor the (true) distribution of $\bmx$ is available.  Throughout the paper, the input variable $\bm{x} \in \mathbb{R}^M$ and the output variable $y \in \mathbb{R}$ are assumed to be \textit{noisy real-valued}, where $M$ is the dimensionality of the input vector. We also assume that queries to get the response $f(\bm{x})$ can be performed cheaply at any $\bm{x}$.

In practice, anomaly attribution is typically coupled with anomaly detection: When we observe a test sample $(\bmx,y)=(\bmx^t,y^t)$, we first compute an anomaly score $a^t = a(\bmx^t,y^t)$ to quantify how anomalous it is. Then, if $a^t \in \mathbb{R}$ is high enough, we go to the next step of anomaly attribution. In this scenario, the task of anomaly attribution is defined as follows. 
\begin{definition}[\textbf{anomaly attribution}]
 Given a black-box regression model $y=f(\bm{x})$, compute the score for each input variable indicative of the extent to which an input variable is responsible for the sample being anomalous.  
\end{definition}
We can readily generalize the problem to that of  \textit{collective} anomaly detection and attribution. Specifically, given a test data set  $\mathcal{D}_\mathrm{test}= \{ (\bm{x}^t,y^t) \mid t=1,\ldots,N_\mathrm{test}\}$, where $t$ is the index for the $t$-th test sample and $N_\mathrm{test}$ is the number of test samples, we can consider an anomaly score as well as attribution score for the whole test set $\calD_{\mathrm{test}}$. We will see later an example where a daily attribution score is computed from 24 hourly observed measurements. 

The standard approach to computing the anomaly score is to use the negative log-likelihood of the test sample(s) (See, e.g.,~\cite{lee2000information,Yamanishi2000,staniford2002practical,noto2010anomaly}). Assume that, from the deterministic regression model, we can somehow obtain $p(y | \bm{x})$, a probability density over $y$ given the input signal $\bm{x}$. Under the i.i.d.~assumption, the anomaly score can be written as
\begin{align}\label{eq:changeScoreDef}
a(\bmx^t,y^t) =  -\ln p(y^t \mid \bm{x}^t), \quad \mbox{or, } \quad
a(\mathcal{D}_\mathrm{test}) =-
\frac{1}{
	N_\mathrm{test}
}\sum_{t \in \mathcal{D}_\mathrm{test}}\ln p(y^t \mid \bm{x}^t) ,
\end{align}
corresponding to the single sample case and collective case, respectively. Obviously, one challenge here is how to estimate $p(y\mid \bm{x})$ from the deterministic regression function. We provide one such approach in Sec.~\ref{subsec:getting_probabilistic_model}. 

Given an anomalous sample $(\bmx^t,y^t)$ and the distribution $p(y \mid \bm{x})$, computing the anomaly score is straightforward. However, computing anomaly \textbf{attribution} score is more challenging. This is in some sense an \textit{inverse problem}: The function $f(\bm{x})$ readily gives an estimate of $y$ from $\bm{x}$, but, in general, there is no obvious way to do the reverse in the \textit{multivariate} case. When an estimate $f(\bm{x}^t)$ looks `bad' in light of an observed $y^t$, what can we say about the contribution, or responsibility, of the respective input variables? Section~\ref{Sec:MOC} provides our proposed answer to this question.

\subsection{Notation}

We use boldface to denote vectors. The $i$-th dimension of a vector $\bm{\delta}$ is denoted as $\delta_i$. The $\ell_1$ and $\ell_2$ norms of a vector are denoted by $\| \cdot \|_1$ and $\| \cdot \|_2$, respectively, and are defined as $\| \bm{\delta} \|_1 \triangleq \sum_i | \delta_i|$ and $\| \bm{\delta} \|_2 \triangleq \sqrt{\sum_i  \delta_i^2}$. The sign function $\mathrm{sign}(\delta_i) $ is defined as being $1$ for $\delta_i>0$, and $-1$ for $\delta_i <0$.
For $\delta_i=0$, the function takes an indeterminate value in $[-1,1]$. For a vector input, the definition applies element-wise, yielding a vector of the same size as the input vector.

We distinguish between a random variable and its realizations via the absence or presence of a superscript. For notational simplicity, we use $p(\cdot)$ as a proxy to represent different probability distributions, whenever there is no confusion. For instance, $p(\bm{x})$ is used to represent the probability density of a random variable $\bm{x}$ while $p(y | \bm{x})$ is a different distribution of another random variable $y$ conditioned on $\bm{x}$. The Gaussian distribution of a scalar variable $y$ is defined as 
\begin{align}\label{eq:1DGaussianDef}
   \mathcal{N}(y \mid m, \sigma^2 ) \triangleq \frac{1}{\sqrt{2\pi\sigma^2}}
   \exp\left\{ -\frac{(y-m)^2}{2\sigma^2} \right\}
\end{align}
where $m$ is the mean and $\sigma^2$ is the variance. The multivariate Gaussian distribution is defined in a similar way.

\section{Limitations of Function-Based Anomaly Attribution}
\label{sec:EIG}

To motivate the likelihood-based attribution approach presented in the next section, this section shows fundamental limitations of the existing function-based attribution approaches: IG, SV, and LIME are inherently deviation-agnostic and are not appropriate for anomaly attribution. For simplicity, we assume for now that the derivative of the black-box regression function $f$ is computable somehow to an arbitrary order. We discuss numerical gradient estimation approaches in Sec.~\ref{subsec:gradient_etimation}.

\subsection{Deviation-agnostic property of integrated gradient}\label{subsec:IG_and_EIG}

\paragraph{Definitions} 
The notion of integrated gradient (IG) was first introduced to the AI community by Sundararajan et al.~\cite{sundararajan2017axiomatic} as a method for axiomatic derivation of an input attribution method. For a test sample at $\bmx^t$, IG of the black-box regression function $f$ for the $i$-th variable is defined by
\begin{align}\label{eq:IG-def}
    \mathrm{IG}_i(\bmx^t \mid\bmx^0) &\triangleq (x_i^t-x_i^0)\int_0^1\mathrm{d}\alpha \ \left.\frac{\partial f}{\partial x_i}\right|_{ \bmx^0 +(\bmx^t-\bmx^0)\alpha},
\end{align}
where $\bmx^0$ is called the baseline input, a parameter representing a ``default'' value of the input. Despite the intimidating look, the integration term simply computes an averaged gradient w.r.t.~$x_i$ on the line from $\bmx^0$ to $\bmx^t$. Hence, $\mathrm{IG}_i(\bmx^t \mid \bmx^0)$ is  $x_i$'s contribution to the increment of $f$ when moving from $\bmx^0$ to $\bmx^t$. Despite the term ``gradient,'' IG is not a gradient but represents the increment (See Fig.~\ref{fig:problem_setting_and_deviation_v_increment} (b) and Table~\ref{table:comparison chart}). This fact becomes clearer if we expand $\partial f/\partial x_i$ into the Taylor series w.r.t.~$\alpha$ and perform integration:
\begin{align}\label{eq:IG_Taylor_expansion}
   \mathrm{IG}_i(\bmx^t \mid\bmx^0) &=
   \frac{\partial f(\bmx^0)}{\partial x_i}\Delta_i
   + \frac{1}{2!}\sum_{j=1}^M \frac{\partial^2 f(\bmx^0)}{\partial x_i\partial x_j} \Delta_i\Delta_j
   + \frac{1}{3!}\sum_{j,k=1}^M \frac{\partial^3 f(\bmx^0)}{\partial x_i\partial x_j\partial x_k} \Delta_i\Delta_j\Delta_k + \ldots,
\end{align}
where we have defined $\Delta_i\triangleq x^t-x^0_i$, etc. The right hand side is simply the collection of differential increments over different orders in the power expansion. Intuitively, the $i$-th attribution score gets large if the $i$-th gradient is large and if the test sample is far from the baseline point along the $i$-th axis.   

As pointed out by~\cite{SippleICML2020}, one of the major issues of IG is the need for the baseline input. Since real-world data may often follow a multi-peaked distribution (see Fig.~\ref{fig:Building_kde_comparison_GBTree_29_highQ.pdf} for an actual example), there may not exist a clearly defined default value. One natural approach for addressing this issue is to integrate out $\bmx^0$ using a distribution $P(\bmx)$: 
\begin{align}\label{eq:EIG-def}
    \mathrm{EIG}_i(\bmx^t) &\triangleq \!\!
    \int\!\mathrm{d}\bmx^0  P(\bmx^0) \mathrm{IG}_i(\bmx^t\mid \bmx^0)
    = \!\!
    \int\!\mathrm{d}\bmx^0  P(\bmx^0)(x_i^t-x_i^0) \int_0^1\!\!\mathrm{d}\alpha  \left.\frac{\partial f}{\partial x_i}\right|_{ \bmx^0 +(\bmx^t-\bmx^0)\alpha}\!\!,
\end{align}
which we call the \textit{expected integrated gradient} (EIG). Obviously, EIG is reduced to IG as a special case when $P(\bmx)$ is chosen to be Dirac's delta function.  When computing EIG, $P(\bmx)$ should ideally be the true distribution or its empirical approximation using the training dataset. Unfortunately, neither of them is  available in our setting (see Table~\ref{table:comparison chart}). 

\paragraph{IG and EIG for deviation} 
In the context of anomaly attribution, we are interested in explaining the deviation $f(\bmx)-y$ rather than $f(\bmx)$ itself. Let us define a new function $F(\bmx,y) \triangleq f(\bmx) - y$ and consider EIG for this function. As an \textit{input} attribution method, the deviation version of EIG, denoted as a two-place function  $\mathrm{EIG}_i(\bmx^t,y^t)$, is defined by
\begin{align}\label{eq:EIG_x_y}
    \mathrm{EIG}_i(\bmx^t,y^t)
    &\triangleq 
    \int\mathrm{d}y^0\int\mathrm{d}\bmx^0\  P(\bmx^0,y^0)\mathrm{IG}_i(\bmx^t,y^t \mid \bmx^0,y^0) 
    \\
    \mathrm{IG}_i(\bmx^t,y^t \mid \bmx^0,y^0) 
    &\triangleq (x_i^t-x_i^0) \int_0^1\mathrm{d}\alpha  \left.\frac{\partial F}{\partial x_i}\right|_{ \bmx^0 +(\bmx^t-\bmx^0)\alpha,\ y^0+(y^t-y^0)\alpha}
\end{align}
for $i=1,\ldots,M$, where $P(\bmx,y)$ is the joint distribution between $\bmx$ and $y$. The following property holds:
\begin{theorem} \label{th:IG_EIG_deviation_agnositic}
IG and EIG are deviation-agnostic.
\end{theorem}
\begin{proof}
   Since $\frac{\partial F}{\partial x_i} = \frac{\partial f}{\partial x_i}$, the statement about IG obviously holds. For EIG, the integration w.r.t.~$y^0$ produces $\int \rmd y^0 P(\bmx^0,y^0) = P(\bmx^0)$, yielding $\mathrm{EIG}_i(\bmx^t,y^t) = \mathrm{EIG}_i(\bmx^t)$. 
\end{proof}
This means that IG and EIG explain what the regression surface looks like at $\bmx^t$ regardless of the nature of the deviation. Hence, they may not be the best approach if our interest is in explaining the deviation. 

\paragraph{Lower-order approximations}
The power expansion approach introduced in Eq.~\eqref{eq:IG_Taylor_expansion} can be used also for EIG. In this case, however, the expansion should be around $\bmx^t$:
\begin{align}\label{eq:f-taylor}
    \left.\frac{\partial f}{\partial x_i}\right|_{ \bmx^0 +\alpha\bm{\Delta}}     \!\!\!\!\!\!\!\!= 
    \frac{\partial f(\bmx^t)}{\partial x_i}
    +(\alpha-1)\sum_{j=1}^M \frac{\Delta_j }{1!}\frac{\partial^2 f(\bmx^t)}{\partial x_i\partial x_j}
+(\alpha-1)^2\sum_{j,k=1}^M \frac{\Delta_j\Delta_k}{2!} \frac{\partial^3 f(\bmx^t)}{\partial x_i\partial x_j \partial x_k} + \ldots,
\end{align}
where $\bm{\Delta} \triangleq \bmx^t-\bmx^0$ and $\Delta_i\triangleq x^t-x^0_i$. We have used the fact $(\bmx^0 +\alpha \bm{\Delta})-\bmx^t = (\alpha-1)\bm{\Delta}$. This expansion allows performing the integration w.r.t.~$\alpha$ analytically:
\begin{align}\label{eq:EIG-2ndApprox}
    \mathrm{EIG}_i(\bmx^t)  = 
    \langle\Delta_i\rangle 
     \frac{\partial f(\bmx^t)}{\partial x_i} 
    - \sum_{j=1}^M \frac{\langle \Delta_i\Delta_j \rangle}{2!} 
    \frac{\partial^2 f(\bmx^t)}{\partial x_i\partial x_j} 
    + \sum_{j,k=1}^M \frac{\langle\Delta_i \Delta_j\Delta_k\rangle }{3!}\frac{\partial^3 f(\bmx^t)}{\partial x_i\partial x_j \partial x_k} 
- \ldots,
\end{align}
where $\langle \cdots \rangle \triangleq \int \mathrm{d}\bmx \ \cdots P(\bmx)$. The first and the second terms provide the first- and second-order approximations of EIG, respectively. 

\paragraph{Sum rules}
Equations~\eqref{eq:IG_Taylor_expansion} and~\eqref{eq:EIG-2ndApprox} readily lead to the following important property:
\begin{align}\label{eq:EIG's_efficiency}
    \sum_{i=1}^M\mathrm{IG}_i(\bmx^t \mid\bmx^0) = f(\bmx^t) - f(\bmx^0), \quad \quad \sum_{i=1}^M\mathrm{EIG}_i(\bmx^t) = f(\bmx^t) - \left\langle f \right\rangle,
\end{align}
\begin{proof}
    Equation~\eqref{eq:IG_Taylor_expansion} is the same as the power expansion of $f(\bmx)$ around $\bmx^0$ evaluated at $\bmx=\bmx^t$ except for $f(\bmx^0)$, the first term of the Taylor series. Hence, the first equation holds. The same argument applies to Eq.~\eqref{eq:EIG-2ndApprox} to prove the second equation.
\end{proof}
These sum rules allow the interpretation that $\mathrm{IG}_i$ and $\mathrm{EIG}_i$ are the share of the $i$-th variable in the total change $f(\bmx^t) - f(\bmx^0)$ and $f(\bmx^t) - \langle f\rangle$, respectively. In EIG, this sum rule implies that EIG is to contrast the local output $f(\bmx^t)$ at the test point $\bmx^t$ with the global mean. If a certain local distribution at $\bmx = \bmx^t$ is used for $P(\bmx)$, we have $\langle f\rangle \approx f(\bmx^t)$, resulting in a meaningless attribution score. In SV, the corresponding property~\eqref{eq:SV_efficiency} is called the efficiency~\cite{roth1988shapley}.

\subsection{Deviation-agnostic property of Shapley value}\label{subsec:SV+}

\paragraph{Definition}
The Shapley value (SV) is one of the most popular attribution metrics in the AI community. There are a few different versions in SV in the literature, depending on how the absence of variables is defined. Here we adopt the definition of the conditional expectation SV~\cite{sundararajan2020many}:
\begin{align}\label{eq:SV_def} 
\mathrm{SV}_i(\bm{x}^t) &= \frac{1}{M}
\sum_{k=0}^{M-1} 
\binom{M-1}{k}^{-1}
\sum_{\mathcal{S}_i: |\calS_i|=k} \left[
\langle f \mid x_i^t, \bm{x}_{\mathcal{S}_i}^t \rangle
- \langle f \mid \bm{x}_{\mathcal{S}_i}^t \rangle
\right].
\end{align}
Due to the combinatorial nature, this definition appears rather complicated. Here $\calS_i$ denotes any subset of the variable indices $i \in \{1,\ldots,M\}$ that does not include $i$ and $|\mathcal{S}_i|$ is its size. The second summation runs over all possible choices of $\mathcal{S}_i$ under the constraint $|\mathcal{S}_i|=k$ from the first summation. We also define the complement $\bar{\mathcal{S}}_i$, which is the subset of $\{1,\ldots,M\}$ excluding $i$ and $\calS_i$. For example, if $M=12, i=3$ and  $\mathcal{S}_i = \{1,2 \}$, the complement $\bar{\mathcal{S}}_i$ will be $\{4,5,\ldots,12\}$. Corresponding to this division, we rearrange the $M$ variables as $\bmx = ( x_i, \bm{x}_{\mathcal{S}_i}, \bm{x}_{\bar{\mathcal{S}}_i} )$. 

In Eq.~\eqref{eq:SV_def}, the prefactor $\frac{1}{M}$ is there to average over the possible choices of $|\calS_i|$. Similarly, the binomial coefficient $\binom{M-1}{|\mathcal{S}_i|}^{-1}$ is to average over all the choices of $\mathcal{S}_i$, which is given by the number of combinations of choosing $|\mathcal{S}_i|$ variables from the $M-1$ variables excluding $i$. This means that SV is essentially the average of $\left[\langle f \mid x_i^t, \bm{x}_{\mathcal{S}_i}^t \rangle - \langle f \mid \bm{x}_{\mathcal{S}_i}^t \rangle\right]$, where
\begin{align}\label{eq:SV_expectation1}
\langle f  \mid x_i^t, \bm{x}_{\mathcal{S}_i}^t \rangle
&\triangleq \int\!\! \mathrm{d}\bm{x} \;
P(\bm{x}) f( x_i=x_i^t, \bm{x}_{\mathcal{S}_i}=\bm{x}^t_{\mathcal{S}_i}, \bm{x}_{\bar{\mathcal{S}}_i}),
\\ \label{eq:SV_expectation2}
\langle f  \mid \bm{x}^t_{\mathcal{S}_i} \rangle 
&\triangleq \int\!\! \mathrm{d}\bm{x} \; P(\bm{x})f(x_i, \bm{x}_{\mathcal{S}_i}=\bm{x}_{\mathcal{S}_i}^t, \bm{x}_{\bar{\mathcal{S}}_i} ).
\end{align}
Here $P(\bmx)$ is the true distribution of $\bmx$, which is not available in our setting (see Table~\ref{table:comparison chart}). In Eq.~\eqref{eq:SV_expectation1}, the integration is reduced to the expectation over the marginal distribution of $\bmx_{\Bar{\calS}_i}$. Note that these quantities have to capture some of the global properties of the data generating mechanism. If $P(\bmx)$ were a localized distribution at~$\bmx^t$, the difference would simply be zero.

\paragraph{SV for deviation} 

Similarly to $\mathrm{EIG}_i(\bmx^t,y^t)$ in Eq.~\eqref{eq:EIG_x_y}, we define $\mathrm{SV}_i(\bmx^t,y^t)$ for the function $F(\bmx,y) = f(\bmx) - y$. As an \textit{input} attribution method, we need to replace  Eqs.~\eqref{eq:SV_expectation1} and~\eqref{eq:SV_expectation2} with
\begin{align}
\langle F \mid x_j^t, \bm{x}_{\mathcal{S}_j}^t, y^t \rangle
&\triangleq \int \mathrm{d}y \int \mathrm{d}\bm{x} \;
P(\bm{x},y) F( x_i=x_i^t, \bm{x}_{\mathcal{S}_i}=\bm{x}^t_{\mathcal{S}_i}, \bm{x}_{\bar{\mathcal{S}}_i}, y = y^t),
\\
\langle F \mid \bm{x}^t_{\mathcal{S}_j}, y^t \rangle
&\triangleq \int \mathrm{d}y \int \mathrm{d}\bm{x} \;
P(\bm{x},y) F( x_i, \bm{x}_{\mathcal{S}_i}=\bm{x}^t_{\mathcal{S}_i}, \bm{x}_{\bar{\mathcal{S}}_i}, y = y^t),
\end{align}
respectively, to get $\mathrm{SV}_i(\bmx^t,y^t)$. Again, the following property holds:
\begin{theorem}\label{th:SV_deviation_agnositic}
SV is deviation-agnostic.
\end{theorem}
\begin{proof}
    Since $F$ is linear in $y$, we can easily see that 
    $
\langle F \mid x_j^t, \bm{x}_{\mathcal{S}_j}^t, y^t\rangle
= \langle f \mid x_j^t, \bm{x}_{\mathcal{S}_j}^t \rangle - y^t
$ and $\langle F \mid \bm{x}^t_{\mathcal{S}_j}, y^t \rangle
= \langle f \mid \bm{x}^t_{\mathcal{S}_j} \rangle - y^t$ hold, which implies $\mathrm{SV}_i(\bmx^t,y^t)= \mathrm{SV}_i(\bmx^t)$.
\end{proof}

\paragraph{Sum rule}
Finally, SV meets the condition called the efficiency~\cite{roth1988shapley}:
\begin{gather}\label{eq:SV_efficiency}
    \sum_{i=1}^M \mathrm{SV}_i(\bmx^t) = f(\bmx^t) - \langle f \rangle.
\end{gather}
\begin{proof} See Appendix~\ref{appendix:SV-efficiency}. 
\end{proof}
The efficiency condition is exactly the same as EIG's sum rule in Eq.~\eqref{eq:EIG's_efficiency}. Similarly to the case of EIG, the efficiency condition implies that SV is essentially the share of the differential increment between the local value $f(\bmx^t)$ and the global mean. Hence, $P(\bmx)$ cannot be a local approximation as the one used in LIME.


\subsection{Deviation-agnostic property of LIME}\label{subsec:LIME}

In general, the local linear surrogate modeling approach fits a linear regression model locally to explain a black-box function in the vicinity of a given test sample $(\bm{x}^t,y^t)$. Algorithm~\ref{algo:LIME} summarizes the local anomaly attribution procedure based on this approach to explain the deviation $f(\bmx)-y$. 

\begin{algorithm}[H]
\caption{Local linear surrogate modeling for anomaly attribution}\label{algo:LIME}
\label{alg:LIME}
\begin{algorithmic}[1] 
\Require Regression model $f(\bm{x})$, test point $(\bm{x}^t,y^t)$,  regularization parameter $\nu$.
\State Randomly populate $N_s$ points $\{ \bm{x}^{t[1]}, \ldots, \bm{x}^{t[N_s]}\}$ in the vicinity of $\bm{x}^t$ ($N_s\sim 1000)$.
\State Compute the deviation $z^{t[n]} \triangleq f(\bm{x}^{t[n]}) -y^t$ for all $n$.
\State Fit a linear model $z = \beta_0 + {\bmbeta}^\top \bm{x}$ using $\nu$ on $\{ (\bm{x}^{t[n]} , z^{t[n]}) \mid n=1,\ldots,N_s \}$.
\State \Return ${\bmbeta}$, which is the local attribution score at $(\bm{x}^t,y^t)$.
\end{algorithmic}
\end{algorithm}

In LIME, an $\ell_1$-regularized model is used to get a sparse and thus easy-to-interpret score. Rather surprisingly, despite the modification to fit $f(\bmx)-y$ rather than $f(\bmx)$, the following property holds:
\begin{theorem}\label{th:LIME_deviation_agnositic} LIME is deviation-agnostic. 
\end{theorem}
\begin{proof}
     With $\nu$ being the $\ell_1$ regularization strength, the lasso loss function for LIME is written as
\begin{align*}
\Psi(\bmbeta,\beta_0) &= \frac{1}{N_s}\sum_{n=1}^{N_s} (z^{t[n]} - \beta_0 - \bmbeta^\top \bm{x}^{t[n]})^2 + \nu\| \bmbeta\|_1,
\\
&= \frac{1}{N_s}\sum_{n=1}^{N_s} (f(\bm{x}^{t[n]}) - (y^t+\beta_0) - \bmbeta^\top \bm{x}^{t[n]})^2 + \nu\| \bmbeta\|_1,
\end{align*}
which is equivalent to the lasso objective for LIME with the intercept $y^t + \beta_0$. Since the lasso objective is convex, the solution $\bmbeta$ is unique. With an adjusted intercept, the attribution score $\bmbeta$ remains the same.
\end{proof}

In the local linear surrogate modeling approach, the final attribution score can vary depending on the nature of the regularization term. For theoretical analysis below, we use a generic algorithm by setting $\nu \to 0_+$ in Algorithm~\ref{algo:LIME}, and call the resulting attribution score $\mathrm{LIME}^0_i$ for $i=1,\ldots,M$. As is well-known, $\mathrm{LIME}^0_i$ is a local estimator of $\partial f/\partial x_i$ at $\bmx=\bmx^t$.

\subsection{Unifying LIME and SV into IG}
\label{subsec:EIG_integration}

We showed that (E)IG, SV, and LIME all share the same deviation-agnostic property. We also showed that EIG and SV satisfy the same sum rule. These findings suggest that they may share a common mathematical structure. This is indeed the case, as discussed below.

\paragraph{Relationship between SV and EIG}

The combinatorial definition SV is a major obstacle in getting deeper insights into what it really represents. With that in mind, we look at the definition~\eqref{eq:SV_def} from a somewhat different angle. We again use the Taylor expansion around $\bmx^t$ for $f$ in the integrand of Eqs~\eqref{eq:SV_expectation1} and~\eqref{eq:SV_expectation2}, which leads to 
\begin{align}
&\langle f \mid x_i^t, \bm{x}_{\mathcal{S}_i}^t \rangle
- \langle f \mid \bm{x}_{\mathcal{S}_i}^t \rangle 
= \langle \Delta_i \rangle \frac{\partial f(\bmx^t)}{\partial x_i}
- \frac{1}{2}\langle \Delta_i^2 \rangle \frac{\partial^2 f(\bmx^t)}{\partial x_i^2}
- \sum_{k \in \bar{\calS}_i} \langle \Delta_i\Delta_k \rangle
\frac{\partial^2 f(\bmx^t)}{\partial x_i \partial x_k} - \ldots.
\end{align}
The first and second terms on the r.h.s.~do not depend on the choice of $\calS_i$, given $i$. In the third term, a $k \in \{1,\ldots,M \}$ ($k\neq i$) will not be included in $\bar{\calS}_i$ if it is chosen in ${\calS}_i$. Thus, for a given value of $|\calS_i|$, the total number of appearances of the $k$ in $\sum_{\calS_i}$ is $\binom{M-2}{|\calS_i|}$ because it is the same as the number of combinations of choosing $|\calS_i|$ variables out of the $M-2$ variables excluding the $k$ in addition to the $i$. Using the following identity on the quotient of the binomial coefficients (see, e.g.,~Chap.4 of~\cite{gross2016combinatorial})
\begin{align}
    \sum_{a=0}^{M-2}\binom{M-1}{a}^{-1}\binom{M-2}{a} = \frac{M}{2},
\end{align}
we have the second-order approximation of SV as
\begin{align}\label{eq:SV-2nd}
    \mathrm{SV}_i(\bm{x}^t) &\approx
    \langle \Delta_i \rangle \frac{\partial f(\bmx^t)}{\partial x_i}
- \frac{1}{2}\sum_{k=1}^M \langle \Delta_i \Delta_k\rangle
 \frac{\partial^2 f(\bmx^t)}{\partial x_i \partial x_k},
\end{align}
which is exactly the same as the first two terms in the EIG expansion in Eq.~\eqref{eq:EIG-2ndApprox}. We have just completed the proof of the following theorem, which reveals what is behind the seemingly complicated combinatorial definition of SV in Eq.~\eqref{eq:SV_def}.
\begin{theorem}[Equivalence of SV to EIG]
\label{th:SV=EIG}
The Shapley value is equivalent to the expected integrated gradient up to the second order.
\end{theorem}
To the best of our knowledge, this is the first result directly establishing the correspondence between SV and IG. In Sec.~\ref{sec:experiments}, we empirically show that indeed SV and EIG systematically give similar attribution scores.

\paragraph{Relationship between LIME and EIG}

LIME as a local linear surrogate modeling approach differs from EIG and SV in two regards. First, LIME does not need the true distribution $P(\bmx)$. Instead, it uses a local distribution to populate local samples. Second, LIME is defined as the gradient, not a differential increment. These observations lead us to an interesting question: Is the \textit{derivative of EIG} in the local limit the same as the LIME attribution score?

To answer this question affirmatively, consider a local distribution around $\bmx^t$ in the following form:
\begin{align}
    P_\eta(\bmx^0 \mid \bmx^t) = \calN(\bmx^0 \mid \bmx^t, \eta \sfI_M) \quad \mbox{with} \quad \eta\to 0,
\end{align}
where $\sfI_M$ is the $M$-dimensional identity matrix. With this distribution, we have
\begin{gather*}
    \langle x^0_i - x_i^t \rangle =0, \quad \langle ( x^0_i - x_i^t )( x^0_k - x_k^t ) \rangle =\delta_{i,k}\eta,
\end{gather*}
where $\delta_{i,k}$ is Kronecker's delta function that takes 1 only when $i=k$ and 0 otherwise. Notice that the second order term is proportional to $\eta$ and vanishes as $\eta\to 0$. The other higher-order terms that appear in the power expansion are either zero or vanish as $\eta\to 0$. An immediate consequence from the expression~\eqref{eq:f-taylor} is
\begin{align}
    \lim_{\eta \to 0}\mathrm{EIG}_i (\bmx^t) =0,
\end{align}
which confirms the previous discussion for Eq.~\eqref{eq:EIG's_efficiency}. On the other hand, the derivative of EIG becomes
\begin{align}
    \frac{\partial  \mathrm{EIG}_i(\bmx^t) }{\partial x_i}
    &=\int\rmd\bmx^0\ P_\eta(\bmx^0 \mid \bmx^t)\left[ 
    \frac{\partial f}{\partial x_i} +(x^t_i - x^0_j)\frac{\partial^2 f}{\partial x_i^2}
    \right]_{\bmx^0+(\bmx^t-\bmx^0)\alpha} \to 
    \frac{\partial f(\bmx^t)}{\partial x_i}
\end{align}
as $\eta \to 0$. As this limit is equivalent to $\bmx^0 \to \bmx^t$ in IG in Eq.~\eqref{eq:IG-def}, the l.h.s.~approaches the derivative of IG. Since the local linear surrogate modeling estimates local gradient, we have just proved the following theorem:
\begin{theorem}[LIME and IG] \label{th:LIME=derivative_of_EIG}
The derivative of IG and EIG is equivalent to LIME:
    \begin{align}
     \mathrm{LIME}^0_i(\bmx^t) = \lim_{\eta \to 0}\frac{\partial \mathrm{EIG}_i (\bmx^t) }{\partial x_i}
     = \lim_{\bmx^0 \to \bmx^t}\frac{\partial \mathrm{IG}_i (\bmx^t \mid \bmx^0) }{\partial x_i},
\end{align}
where $P(\bmx^0 )= \calN(\bmx \mid \bmx^t, \eta \sfI_M)$ is used in the definition of EIG in Eq.~\eqref{eq:EIG-def}. 
\end{theorem}

Since EIG, SV, and LIME can be derived from or associated with IG as shown above, it is legitimate to say that they are in the \textit{integrated gradient family}.

\subsection{Summary of limitations}

We have shown that IG, EIG, SV, and LIME are deviation-agnostic in Theorems~\ref{th:IG_EIG_deviation_agnositic},~\ref{th:SV_deviation_agnositic}, and~\ref{th:LIME_deviation_agnositic}. Since our goal is to provide an actionable explanation on a detected anomaly, this can be a serious issue. We have also shown that SV and LIME are derived from or associated with IG in Theorems~\ref{th:SV=EIG} and~\ref{th:LIME=derivative_of_EIG}. As suggested by IG's Taylor series representation in Eq.~\eqref{eq:IG_Taylor_expansion}, the attribution score of the IG family is essentially the differential increment of $f(\bmx)$ when going from the baseline input $\bmx^0$ to the test point $\bmx^t$ (see Fig.~\ref{fig:problem_setting_and_deviation_v_increment} (b)). As already discussed, the biggest issue here is the arbitrariness of the baseline input. EIG and SV attempt to neutralize it by integrating out $\bmx^0$ in exchange for the demanding requirement on the availability of the global distribution $P(\bmx)$.

These two issues -- the deviation-agnostic property and the explicit or implicit dependency on the arbitrary baseline input -- are inherent to the IG family. A key idea in the proposed framework is to leverage the point that gives the highest possible likelihood in the vicinity of a test sample as the reference point for attribution (c.f.~Table~\ref{table:comparison chart}). In the next section, we will show that this idea indeed succeeds in eliminating these issues.

\section{Likelihood Compensation}\label{Sec:MOC}

This section presents a novel \textit{likelihood-based} framework for anomaly attribution.

\subsection{Seeking reference point through likelihood}
\label{subsec:seeking_reference_point}

\paragraph{Definition of LC}
In a typical anomaly detection scenario, samples in the training dataset are assumed to have been collected under normal conditions, and hence, the learned function $y=f(\bmx)$ represents normalcy as well. As discussed in Sec.~\ref{sec:problem_setting}, the canonical measure of anomalousness is negative log likelihood $-\ln p(y\mid \bmx)$. A low likelihood value signifies anomaly, and vice versa. From a geometric perspective, on the other hand, being an anomaly implies deviating from a certain normal value. We are interested in integrating these two perspectives. 

Suppose we just observed a test sample $(\bmx^t,y^t)$ being anomalous because of a low likelihood value. Given the regression function $y=f(\bmx)$, there are two possible geometric interpretations on the anomalousness (see Fig.~\ref{fig:LC_convergence.pdf}). One is to start with the input $\bmx = \bmx^t$, and observe the deviation $f(\bmx^t)-y^t$. In some sense, $(\bmx,y)=(\bmx^t,f(\bmx^t))$ is a reference point against which the observed sample $(\bmx^t,y^t)$ is judged. The other is to start with the output $y=y^t$, and move horizontally, looking for a perturbation $\bmdelta$ such that $\bmx = \bmx^t + \bmdelta$ gives the maximum possible fit to the normal model. In this case, the reference point is $(\bmx^t + \bmdelta, y^t)$ and $\bmdelta$ is a ``horizontal deviation.'' Since $\bmdelta$ is supposed to be zero if the sample is perfectly normal, each component $\delta_1,\ldots,\delta_M$ can be viewed as a value indicative of the responsibility of each input variable. 

Based on the intuition above, we propose a new \textit{likelihood-based} attribution scoring framework with the following optimization problem: 
\begin{align}\label{eq:LC_definition_1sample_general}
    \bmdelta^* = \arg \max_{\bmdelta}\left\{ 
    \ln p(y^t \mid \bmx^t + \bmdelta) 
    \right\} \quad  \mbox{subject to } \quad \bmx^t+ \bmdelta \in \mathrm{vic}(\bmx^t),
\end{align}
where $\mathrm{vic}(\bmx^t)$ reads ``in the vicinity of $\bmx^t$.'' We call the $\bmdelta^*$ the \textbf{likelihood compensation} (LC), as it compensates for the loss in likelihood incurred by the anomalous prediction. Notice that $\bmdelta^*$ is defined through $p(y \mid \bmx)$, and hence, the randomness of $y$ is automatically taken into account. In other words, $y^t$ does not have to be absolutely correct.

In the collective anomaly detection/attribution case corresponding to $a(\calD_{\mathrm{test}})$ in Eq.~\eqref{eq:changeScoreDef}, the LC score is defined as
\begin{align}\label{eq:LC_definition_general}
    \bmdelta^* = \arg \max_{\bmdelta}\left\{ \frac{1}{N_{\mathrm{test}}}\sum_{t=1}^{N_{\mathrm{test}}}
    \ln p(y^t \mid \bmx^t + \bmdelta) 
    \right\}  \quad \mbox{subject to }  \quad \bmx^t+\bmdelta \in \mathrm{vic}(\bmx^t),
\end{align}
which obviously includes Eq.~\eqref{eq:LC_definition_1sample_general} as a special case. In the collective case, the resulting attribution is an averaged explanation for  $\calD_{\mathrm{test}}$. For example, it explains what was wrong with yesterday overall, rather than explaining about a specific moment in the day. Section~\ref{subsec:bulding experiments} provides such a real-world scenario.  

The LC score has unique features compared to that of the IG family. First, it is \textit{deviation-sensitive}. This is obvious because the log-likelihood itself is the measure of anomalousness, and the deviation should be reflected in the anomaly score. Second, it is more principled as a solution to anomaly attribution because both anomaly detection and attribution are formalized on a common ground of likelihood. Third, it provides richer information on the input variables. $\bmdelta$ is a correction to the input to get the reference point having the highest possible likelihood in the vicinity of $\bmx^t$, admitting an interpretation like ``the sample would have been normal if the input value had been $\bmx^t+\bmdelta$.'' This is one way of analyzing anomalies with a counterfactual hypothesis, a unique property lacking in the IG family.

\paragraph{Optimization problem for LC}
To compute the LC score, we need to specify the functional form of $p(y \mid \bmx)$. We employ a Gaussian-based observation model
\begin{gather}
\label{eq:obs_model}
p(y^t\mid \bm{x}^t+\bm{\delta}) =
\mathcal{N}(y^t \mid f(\bm{x}^t+\bm{\delta}), \sigma^2(\bm{x}^t)).
\end{gather}
We discuss how to estimate the variance $\sigma^2(\bm{x}^t)$ in Sec.~\ref{subsec:getting_probabilistic_model}. 

In addition, we need to define the vicinity. The vicinity constraint can be incorporated as regularization on $\bmdelta$. This can be problem-specific. For example, if there is an infeasible region in the domain of $\bmx$, the regularization should penalize such a choice of $\bm{\delta}$ that makes $\bm{x}^t+\bm{\delta}$ infeasible. If there are no known constraints in the domain of~$\bmx$, it should be designed to properly address a well-known issue of $\ell_1$-regularization in the original LIME: In the presence of multiple correlated explanatory variables, lasso tends to pick one at random~\cite{roy2017selection}, which can be problematic in attribution. Here, we propose to use the elastic net regularization~\cite{ESL2}. Now the optimization problem on a perturbation $\bmdelta$ is written as: 
\begin{align} \label{eq:LC_Gaussian_elastic_net}
\bm{\delta}^*=\arg
\min_{\bm{\delta}}\left\{ \frac{1}{N_\mathrm{test}}\sum_{t=1}^{N_\mathrm{test}}
\frac{\left[ y^t - f(\bm{x}^t + \bm{\delta})\right]^2}{2 \sigma^2(\bm{x}^t)} 
+\frac{1}{2}\lambda \| \bm{\delta} \|_2^2 + \nu \| \bm{\delta}\|_1
\right\}.
\end{align}
This is the main problem formulation studied in this paper. Note that this includes $N_{\mathrm{test}}=1$ as a special case.

\paragraph{Is Gaussian general enough?} Here we briefly discuss how the Gaussian-based formulation~\eqref{eq:obs_model} does \textit{not} result in much loss of generality.  Clearly, the distribution of $f(\bmx^t)$ is not always Gaussian in general. Notice, however, Eq.~\eqref{eq:obs_model} says that the \textit{deviation} or the \textit{error} $f(\bmx^t)-y^t$ should follow Gaussian. This is exactly the same situation when Carl Friedrich Gauss invented Gaussian-based fitting~\cite{brereton2014normal}: Planetary motions do not follow Gaussian but the error does. See Fig.~\ref{fig:Building_kde_comparison_GBTree_29_highQ.pdf} for a real example in our context.

\paragraph{What if $y^t$ is incorrect?}
We have provided an intuition of LC as the deviation measured horizontally. This may lead to a question of what if $y^t$ is incorrect. As commented below Eq.~\eqref{eq:LC_definition_1sample_general}, our framework views $y$ as a random variable and it does not have to be absolutely correct. In fact, the optimization problem~\eqref{eq:LC_Gaussian_elastic_net} shows that the resulting attribution score $\bmdelta^*$ does depend on the variance of $y$, given $\bmx$. Of course, it is possible that the error in $y^t$ happens to go far beyond the reasonable range assumed in $p(y \mid \bmx)$. In that case, however, the attribution problem itself would be ill-posed to any attribution methods.

\paragraph{Relationship with adversarial training}

The optimization problem of LC~\eqref{eq:LC_definition_1sample_general} can be rewritten as 
\begin{align}
    \min_{\bmdelta:\ \bmx^t+\bmdelta \in \mathrm{vic}(\bmx^t)} \left\langle \mbox{Loss}(\bmx^t + \bmdelta \mid y^t, \bmtheta)\right\rangle,
\end{align}
where $\mathrm{Loss}$ denotes the loss function, which is the negative log likelihood in our case, and $\bmtheta$ is the model parameters that are actually not accessible in our doubly black-box setting. Also, $\langle \cdot \rangle$ denotes empirical average over $\{(\bmx^t,y^t)\}$. This form is reminiscent of the min-max problem in adversarial training~\cite{madry2018towards,qin2019adversarial}:
\begin{align}
    \min_{\bmtheta} 
    \max_{\bmdelta:\ \bmx^t+ \bmdelta \in \mathrm{vic}(\bmx^t)} \left\langle \mbox{Loss}(\bmx^t + \bmdelta \mid y^t, \bmtheta)\right\rangle,
\end{align}
where $y^t$ is typically the class label of the $t$-th training sample, unlike ours. The similarity is obvious, but they are working towards the opposite directions. LC's starting point is that the sample is anomalous, and $\bmdelta$ is to bring it back to a normal point. In contrast, adversarial training assumes the samples are normal, and $\bmdelta$ is to make the normal sample as adversarial as possible by changing the output significantly. Also a requirement for an example to be adversarial is that the change should be imperceptible, which is not important in our case.

\begin{figure}[t]
\begin{center}
\includegraphics[trim={3.5cm 1.5cm 1.5cm 3cm},clip,width=10cm]{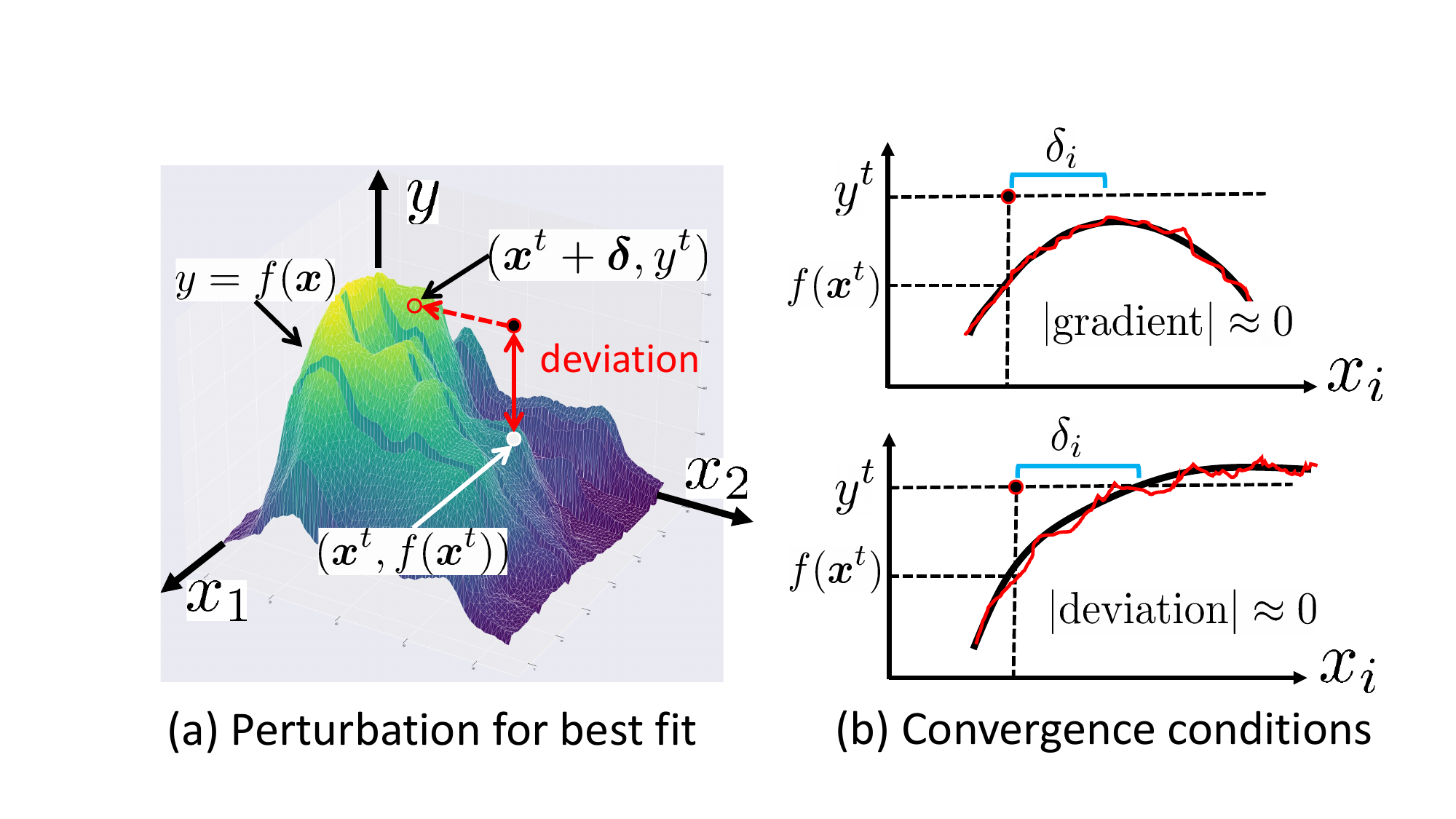}
\end{center}
\caption{Illustration of likelihood compensation (LC). (a)~For a given test sample $(y^t,\bm{x}^t)$, LC seeks a perturbation~$\bmdelta$ that achieves the best possible fit with the black-box regression model $f(\bm{x})$. (b) The iterative updates Eqs.~\eqref{eq:delta_solution}-\eqref{eq:phi_delta} converge when the deviation or the (smoothed) gradient vanishes. See Sec.~\ref{subsec:optimization} for more details. }
\label{fig:LC_convergence.pdf}
\end{figure}

\subsection{Deriving probabilistic prediction model}
\label{subsec:getting_probabilistic_model}

So far we have assumed the predictive distribution $p(y\mid \bm{x})$ is given as Eq.~\eqref{eq:obs_model}, i.e., 
\begin{align}\label{eq:p(y|x)_Gaussian}
    p(y\mid \bmx)=\calN(y \mid f(\bmx), \sigma^2(\bmx)).
\end{align}
In this Gaussian observation model, the only parameter to be estimated is $\sigma^2(\bmx)$. If there are too few test samples, we have no choice but to set $\sigma^2(\bmx)$ to a constant using prior knowledge. Otherwise, we can obtain an estimate of $\sigma^2(\bmx)$ using a subset of $\mathcal{D}_\mathrm{test}$ in a cross-validation (CV)-like fashion as follows. 

Let $\mathcal{D}^t_\mathrm{ho}=\{(\bm{x}^{(n)}, y^{(n)}) \mid n=1,\ldots, N_\mathrm{ho} \} \subset \mathcal{D}_\mathrm{test}$ be a held-out (`ho') data set that does not include a given test sample $(\bm{x}^t,y^t)$. Here $N_\mathrm{ho}$ is the number of samples in it. For the observation model Eq.~\eqref{eq:obs_model} and the test sample $\bm{x}^t$, we consider a locally weighted version of maximum likelihood:
\begin{align}\label{eq:ML_for_sigma2}
&\max_{\sigma^2} \sum_{n=1}^{N_\mathrm{ho}} w_n(\bm{x}^t)
\ln p(y^{(n)}\mid \bmx^{(n)}) 
=\max_{\sigma^2}
\sum_{n=1}^{N_\mathrm{ho}} w_n(\bm{x}^t)
\left\{\ln\frac{1}{\sqrt{2\pi\sigma^2}} - \frac{(y^{(n)}  - f(\bm{x}^{(n)})  )^2}{2\sigma^2}\right\},
\end{align}
where $w_n(\bm{x}^t)$ is the similarity between $\bm{x}^t$ and $\bm{x}^{(n)}$. A reasonable choice of this is
\begin{align}\label{eq:Gaussian_kernel}
w_n(\bm{x}^t)= w_0 + \exp\left\{-\frac{1}{2\eta_0^2}\| \bmx^{(n)} - \bmx^t \|^2 \right\},
\end{align}
where $w_0$ and $\eta^2$ are constants. The maximizer of Eq.~\eqref{eq:ML_for_sigma2} can be easily found by taking the derivative w.r.t.~$\sigma^{-2}$. The solution is given by
\begin{align}\label{eq:sigma2_for_test_samples}
\sigma^2(\bmx^t) =  
\sum_{n=1}^{N_\mathrm{ho}}
\frac{w_n(\bm{x}^t)}{\sum_{m} w_m(\bm{x}^t)}
 \left[y^{(n)}  - f(\bm{x}^{(n)})  \right]^2.
\end{align}
This has to be computed for each $\bm{x}^t \in \mathcal{D}_\mathrm{test}$. When LC scores are compared over different $t$'s, too much variability in $\sigma^2(\bmx^t)$ tends to obfuscate meaningful signals. For standardized data, $\eta_0=1$ and $w_0 \gtrsim  5$ would be a reasonable choice.

\subsection{Deriving updating equation}
\label{subsec:optimization}

Although seemingly simple, solving the optimization problem~\eqref{eq:LC_Gaussian_elastic_net} is generally challenging. Due to the black-box nature of $f$, we do not have access to the parametric form of $f$, let alone the gradient. In addition, as is the case in deep neural networks, $f$ can be non-smooth (see the red curves in Fig.~\ref{fig:LC_convergence.pdf}~(b)), which makes numerical estimation of the gradient tricky.

To derive an optimization algorithm, we first note that there are two origins of non-smoothness in the objective function in~\eqref{eq:LC_Gaussian_elastic_net}. One is inherent to $f$ while the other is due to the added $\ell_1$ penalty. To separate them, let us denote the objective function in Eq.~\eqref{eq:LC_Gaussian_elastic_net} as $J(\bm{\delta})+\nu\|\bm{\delta} \|_1$, where
\begin{align}
    J(\bmdelta) \triangleq 
    \frac{1}{N_\mathrm{test}}\sum_{t=1}^{N_\mathrm{test}}
\frac{\left[ y^t - f(\bm{x}^t + \bm{\delta})\right]^2}{2 \sigma^2_t} 
+\frac{1}{2}\lambda \| \bm{\delta} \|_2^2.
\end{align}
Since we are interested only in a local solution in the vicinity of $\bm{\delta}=\bm{0}$, it is natural to adopt an iterative update algorithm starting from $\bm{\delta}\approx \bm{0}$. Suppose that we have an estimate $\bm{\delta}=\bm{\delta}^\mathrm{old}$ that we wish to update. If we have a reasonable approximation of the gradient in its vicinity, denoted by $\llangle \nabla J(\bm{\delta}^\mathrm{old}) \rrangle$, the next estimate can be found by
\begin{align}\label{eq:prox_Gradient_eq}
\bm{\delta}^\mathrm{new}=
\arg\min_{\bm{\delta}} &\left\{ J(\bm{\delta}^\mathrm{old})+
(\bm{\delta}-\bm{\delta}^\mathrm{old})^\top \llangle\nabla J(\bm{\delta}^\mathrm{old})\rrangle
+\frac{1}{2\kappa}\|\bm{\delta}- \bm{\delta}^\mathrm{old} \|_2^2
+ \nu \| \bm{\delta}\|_1
\right\}
\end{align}
in the spirit of the proximal gradient~\cite{parikh2014proximal}, where $\kappa$ is a hyperparameter representing the learning rate. Notice that the first three terms in the curly bracket correspond to a second-order approximation of $J(\bm{\delta})$ in the vicinity of $\bm{\delta}^\mathrm{old}$. We find the best estimate under this approximation.

Fortunately, the r.h.s.~has an analytic solution. By differentiating the objective in Eq.~\eqref{eq:prox_Gradient_eq} w.r.t.~$\bmdelta$ and equating the result to $\bmzero$, we have the condition of optimality as
\begin{align}\label{eq:optimality_delta}
\bm{\delta} = \bm{\phi} - \kappa \nu \ \mathrm{sign}(\bm{\delta}), \quad \quad \mbox{where}\quad \bm{\phi} \triangleq \bm{\delta}^\mathrm{old} -\kappa \llangle \nabla J(\bm{\delta}^\mathrm{old})\rrangle.
\end{align}
For the $i$-th dimension, if $\phi_i > \kappa\nu$ holds, we have ${\phi}_i \pm \kappa >0$ and thus $\phi_i - \kappa \nu \ \mathrm{sign}(\delta_i)$ must be positive. By setting $\mathrm{sign}(\delta_i)=1$, we conclude $\delta_i =  \phi_i - \kappa \nu$ in this case.  Similar arguments easily verify $\delta_i = \phi_i + \kappa \nu$ for  $\phi_i < - \kappa\nu$. An interesting situation arises when $|\phi_i| \leq \kappa \nu$. Remember that the sign function takes an indeterminate value within $[-1,1]$ at zero. If $\delta_i>0$ is assumed, $\mathrm{sign}(\delta_i) = +1$ and the r.h.s.~of Eq.~\eqref{eq:optimality_delta} must be $\phi_i - \kappa \nu$, which is negative and contradicts the positivity assumption. Thus, the only possible choice is $\delta_i =0$. To summarize, the solution of Eq.~\eqref{eq:prox_Gradient_eq} is given by
\begin{align}\label{eq:delta_solution}
    \delta_i =
    \begin{cases}
   \phi_i - \kappa\nu, &\phi_i > \kappa\nu \\
    0, &|\phi_i| \leq \kappa\nu \\
     \phi_i + \kappa\nu, &\phi_i < - \kappa\nu
    \end{cases}.
\end{align}
Performing differentiation, we see that $\bm{\phi}$ is given by
\begin{align}\label{eq:phi_delta}
\bm{\phi} &= (1 - \kappa\lambda)\bm{\delta}^\mathrm{old} +
\kappa
    \frac{1}{N_\mathrm{test}}\sum_{t=1}^{N_\mathrm{test}}
    \left\{
\frac{y^t - f(\bm{x}^t + \bm{\delta}) }{\sigma^2_t }
\right\}
\left\llangle
\frac{\partial f(\bm{x}^t+\bm{\delta})}{\partial \bm{\delta}} \right\rrangle.
\end{align}
Note that $f(\bm{x}^t + \bm{\delta})$ is readily available at any $\bm{\delta}$ without approximation. 

Here we provide some intuition behind the updating equation~\eqref{eq:phi_delta}. Convergence is achieved when either the deviation $y^t - f$ or the gradient $\llangle \partial f/\partial \bm{\delta}\rrangle$ vanishes at $\bm{x}^t + \bm{\delta}$. These situations are illustrated in Fig.~\ref{fig:LC_convergence.pdf}~(b). As shown in the figure, $\delta_i$ corresponds to the ``\textit{horizontal deviation}'' along the $x_i$ axis between the test sample and the regression function. If there is no horizontal intersection on the regression surface it seeks the zero gradient point based on a smooth surrogate of the gradient.

\subsection{Estimating smooth gradient}
\label{subsec:gradient_etimation}

The final step is to estimate the smooth surrogate of the gradient $\llangle \partial f/\partial \bm{\delta} \rrangle$ in the vicinity of $\bmx_\delta \triangleq \bmx^t + \bmdelta$. To handle the potential non-differentiability of $f$, we define the gradient as the local mean of the slope function $[f(\bmx_\delta + h\bme_i) - f(\bmx_\delta)]/h$, where $h$ is a small perturbation and $\bme_i$ is a unit vector which assumes value 1 in the $i$-th entry and 0 otherwise. Let $p(h \mid \bmx_\delta)$ be an assumed local distribution for $h$ around $\bmx_\delta$. Now we define the local gradient as  
\begin{align}\label{eq:gradient_estimation_as_mean_slope}
    \left\llangle \frac{\partial f(\bmx_\delta)}{\partial \delta_i} 
    \right\rrangle
    &\triangleq
    \int \rmd h \ p(h \mid \bmx_\delta) \frac{f(\bmx_\delta + h\bme_i) - f(\bmx_\delta)}{h}
    \approx  \frac{1}{N_s}\sum_{n=1}^{N_s}\frac{f(\bmx_\delta + h^{[n]}\bme_i) - f(\bmx_\delta)}{h^{[n]}},
\end{align}
where $h^{[n]}$ is the $n$-the sample from $p(h \mid \bmx_\delta)$ and $N_s$ is the number of perturbations generated. The second approximate equality is due to Monte Carlo approximation, which is guaranteed to converge as $N_s\to \infty$. One reasonable choice for the local distribution is $p(h \mid \bmx_\delta)=\calN(h \mid \bmx_\delta,\eta^2)$ with $\eta^2$ being the standard deviation of the perturbation. In this case, perturbations that happen to be zero numerically have to be excluded from the computation.

\paragraph{Alternative approaches to smooth gradient estimation}
It is worth noting that local gradient estimation has been studied in evolutionary computation for years~\cite{salomon1998evolutionary,salomon2009evolutionary}. The key idea is to leverage the notion of Gaussian smoothing of a potentially non-continuous function:
\begin{align}
    f_\eta(\bmx_\delta) &\approx \int\rmd\bmh\ \calN(\bmh \mid  \bmzero, \eta^2\sfI_M) f(\bmx_\delta + \bmh),
\end{align}
where $\sfI_M$ is the $M$-dimensional identity matrix. As $f_\eta$ can be viewed as a locally smoothed version of $f$, the gradient of $f_\eta$ is a reasonable estimate of $\llangle \partial f/\partial \bm{\delta} \rrangle$. By using integration by parts and Monte Carlo estimation, the Gaussian smoothing approach gives
\begin{align}\label{eq:gradient_estimation_Gaussian_smoothing}
    \frac{\partial f_\eta(\bmx_\delta)}{\partial x_i} \approx \frac{1}{N_s} \sum_{n=1}^{N_s}\frac{h_i^{[n]}}{\eta^2 }\left[ f(\bmx_\delta + h_i^{[n]}\bme_i) - \bar{f}(\bmx_\delta) \right], \quad \bar{f}(\bmx_\delta) \triangleq\frac{1}{N_s}\sum_n f(\bmx_\delta + \bme_i h^{[n]}), 
\end{align}
which is another reasonable estimate of the local gradient. From the perspective of numerical computation, however, there is no compelling reason to use this expression instead of~\eqref{eq:gradient_estimation_as_mean_slope} because~\eqref{eq:gradient_estimation_Gaussian_smoothing} tends to have a much larger variance than~\eqref{eq:gradient_estimation_as_mean_slope}.

Another reasonable approach is to locally fit a linear function and use the coefficients as a surrogate of the gradient, as proposed by~\cite{Ide21AAAI}. This provides an estimate as accurately as the direct slope estimation approach~\eqref{eq:gradient_estimation_as_mean_slope} does. However, one issue is that the resulting estimation formula is not linear in $f$ and hence is not eligible for fast vectorized computation. This can be problematic when repeated gradient estimation is required.

\begin{algorithm}[tb]
\caption{Likelihood Compensation}\label{algo:OC}
\label{alg:algorithm}
\begin{algorithmic}[1] 
\Require Black-box regression model $f(\bm{x})$, test data $\mathcal{D}_\mathrm{test}$, and parameters $\lambda,\nu,\kappa$.
\For{all $\bm{x}^t \in  \mathcal{D}_\mathrm{test}$ }
	\State Compute  $\sigma^2_t$ with Eq.~\eqref{eq:sigma2_for_test_samples}.
\EndFor
\State Randomly initialize $\bm{\delta}\approx \bm{0}$.
\Repeat
\State Set $\bm{g}=\bm{0}$.
\For{all $\bm{x}^t \in  \mathcal{D}_\mathrm{test}$ }
	\State Compute $\left\llangle
\frac{\partial f(\bm{x}^t+\bm{\delta})}{\partial \bm{\delta}} \right\rrangle$ with Eq.~\eqref{eq:gradient_estimation_as_mean_slope}
	\State Update $\bm{g} \leftarrow \bm{g} + \left\llangle
\frac{\partial f(\bm{x}^t+\bm{\delta})}{\partial \bm{\delta}} \right\rrangle\frac{y^t - f(\bm{x}^t+\bm{\delta}) }{N_{\mathrm{test}} \sigma^2_t }$.
\EndFor
\State $\bm{\phi} = (1 - \kappa\lambda)\bm{\delta} + \kappa \bm{g}$.
\State Find $\bm{\delta}$ with Eq.~\eqref{eq:delta_solution}.
\Until convergence.
\State \Return $\bm{\delta}$
\end{algorithmic}
\end{algorithm}

\subsection{Algorithm Summary}
\label{subsec:Algo_summary}

Algorithm~\ref{algo:OC} summarizes the iterative procedure for finding $\bm{\delta}$. The most important parameter is the $\ell_1$ regularization strength $\nu$, which has to be hand-tuned depending on the business requirements of the application of interest. On the other hand, the $\ell_2$ strength $\lambda$ controls the overall scale of $\bm{\delta}$. It can be fixed to some value between 0 and 1. In our experiments, it was adjusted so its scale is on the same order as LIME's output for consistency.
It is generally recommended to rescale the input variables to have zero mean and unit variance before starting the iteration (assuming $N_\mathrm{test}\gg 1$), and retrieve the scale factors after convergence. 
For the learning rate $\kappa$, in our experiments, we fixed $\kappa = 0.1$ and shrank it (geometrically) by a factor of 0.98 in every iteration.

In addition to the parameters listed in Algorithm~\ref{algo:OC}, gradient estimation by Eq.~\eqref{eq:gradient_estimation_Gaussian_smoothing} requires two minor parameters, $N^\mathrm{s},\eta$. In our experiments, we used $N^\mathrm{s} = 10$, which was confirmed to provide sufficient convergence.

\section{Experiments}\label{sec:experiments}
This section presents empirical evaluation of the proposed anomaly attribution framework. For comprehensive coverage, we use five datasets with different statistical complexities, including one from a real business use case. The goals of this evaluation are to 1) provide a clear picture of what deviation-sensitivity of an attribution method buys us using a simple synthetic model; 2) demonstrate LC's capability of  providing directly interpretable attribution scores; 3) point out inherent issues with the IG family in anomaly attribution; 4) quantitatively analyze the consistency and inconsistency among different attribution methods; and 5) demonstrate how LC was able to make a difference in a real business scenario. 
For the reader's convenience, we summarize the datasets we used in Table~\ref{table:datasets}.

\begin{table}[htb]
\label{table:datasets}
\caption{Summary of the datasets used. `NA' denotes `not available.'}
\footnotesize
\begin{tabular}{llllll}
\hline \hline
                & $N_{\mathrm{train}}$ & $N_{\mathrm{test}}$ & $M$ & $f(\bmx)$                   & characteristics                                                        \\ \hline
2D sinusoidal   & $\infty$          & 1                   & 2   & analytic                & smooth, periodic, noise-free                                           \\
Boston Housing        & 506                  & 1                   & 13  & RF           & multimodal, clustered, partially discontinuous                         \\
California Housing     & $20\,640$            & 1                   & 9   & GBT & large, uni- or bimodal, partially discontinuous \\
Diabetes        & 442                  & 1                   & 10  & DNN                     & unimodal, semi-discrete                                                \\
Building Energy & NA                   & 24                  & 12  & commercial         & noisy, periodic, online, real building HVAC  \\
\hline 
\end{tabular}
\end{table}

\subsection{Baselines}\label{subsec:baselines}

We compare LC with five possible alternatives\footnote{The Python implementation will be made available upon acceptance of the paper.}: LIME~\cite{ribeiro2016should}, SV~\cite{vstrumbelj2014explaining}, IG~\cite{SippleICML2020}, EIG as defined in Eq.~\eqref{eq:EIG-def}, and the $Z$-score, as summarized in Table~\ref{table:comparison chart}. For anomaly attribution, LIME, SV, IG, and EIG are applied to the deviation $f(\bmx) - y $ rather than $f(\bmx)$. The $Z$-score is one of the standard univariate outlier detection metrics in the unsupervised setting, and defined as $ Z_i \triangleq (x_i^t-m_i)/\sigma_i$ for the $i$-th variable, where $m_i,\sigma_i$ are the mean and the standard deviation of $x_i$, respectively. In SV, we used the same sampling scheme as that proposed in~\cite{vstrumbelj2014explaining} to handle combinatorial complexity. In IG and EIG, we used the trapezoidal rule with 100 equally-spaced intervals to perform the integration w.r.t.~$\alpha$. In IG, EIG, and LC, we used the same gradient estimation algorithm in Eq.~\eqref{eq:gradient_estimation_as_mean_slope}.

We listed SV, EIG, and the $Z$-score here for comparison purposes despite the fact that they are not actually applicable in our doubly black-box setting (see Table~\ref{table:comparison chart}). We excluded other contrastive and counterfactual methods such as \cite{wachter2017counterfactual} as they require white-box access to the model and are predominantly used in classification settings. In the real-world case study in Sec.~\ref{subsec:bulding experiments}, we validated our approach with feedback from domain experts as opposed to crowd-sourced user studies with lay users. In industrial applications, the end-user's interests can be highly specific to the particular business needs and the system's inner workings tend to be difficult for non-experts to understand and simulate.

\begin{figure}[tb]
\centering
\includegraphics[trim={0.2cm 0.cm 0cm 0.4cm},clip,width=8cm]{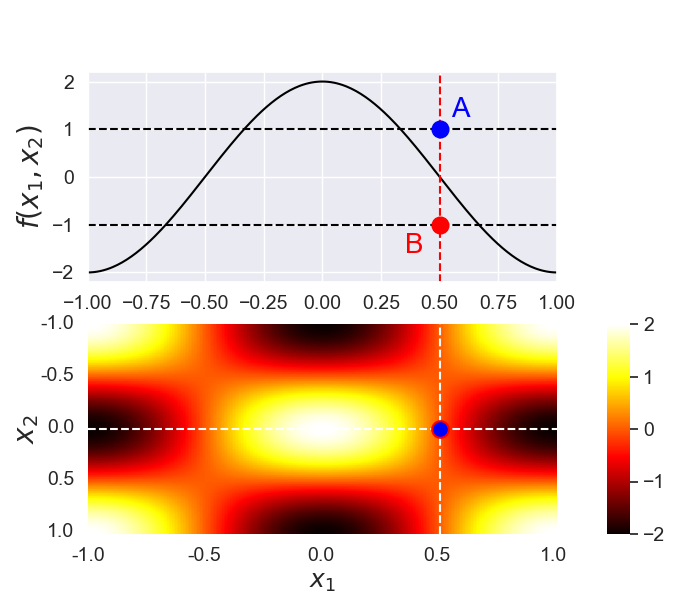}
\vspace{-3mm}
\caption{2D Sinusoidal Curve with the $x_2=0$ slice on the top. The points A and B are at $y^t=1$ and $-1$, respectively, while they are at the same $\bmx^t =(0.5,0)$. }
\label{fig:2D_sinusoidal}
\end{figure}

\subsection{Deviation-sensitivity} 

\paragraph{2D Sinusoidal Curve data}
To illustrate LC's deviation-sensitive property, we computed attribution scores for a regression curve defined by a two-dimensional (2D) sinusoidal function
\begin{align}
    f(\bmx) &= 2 \cos (\pi x_1)\sin(\pi x_2).
\end{align}
We defined two test points: A is at $(\bmx^t,y^t)=((0.5,0),+1)$ and B is at $((0.5,0),-1)$, as shown in Fig.~\ref{fig:2D_sinusoidal}. Computation of SV, EIG, and the $Z$-score needs the true distribution $P(\bmx)$. We used the empirical approximation for $P(\bmx)$ by randomly generating samples with the uniform distribution in $[-4,4]^2$, which was set to be wide enough to simulate sparse sample distributions, as opposed to Gaussian-like unimodal distributions. We created 10 held-out datasets, each of which consists of $N_{\mathrm{ho}}=100$ samples. Those data sets were used also to compute the mean and standard deviation for the $Z$-score. For IG, we gave two baseline inputs: $(0,0)$ and $(0,1)$, resulting in two IG scores for each dataset. The regularization parameters are set to be negligible values as regularization is unimportant in this low-dimensional setting.

\begin{figure}[bt]
\centering
\includegraphics[trim={0.cm 0.cm 0cm 0.cm},clip,width=7.6cm]{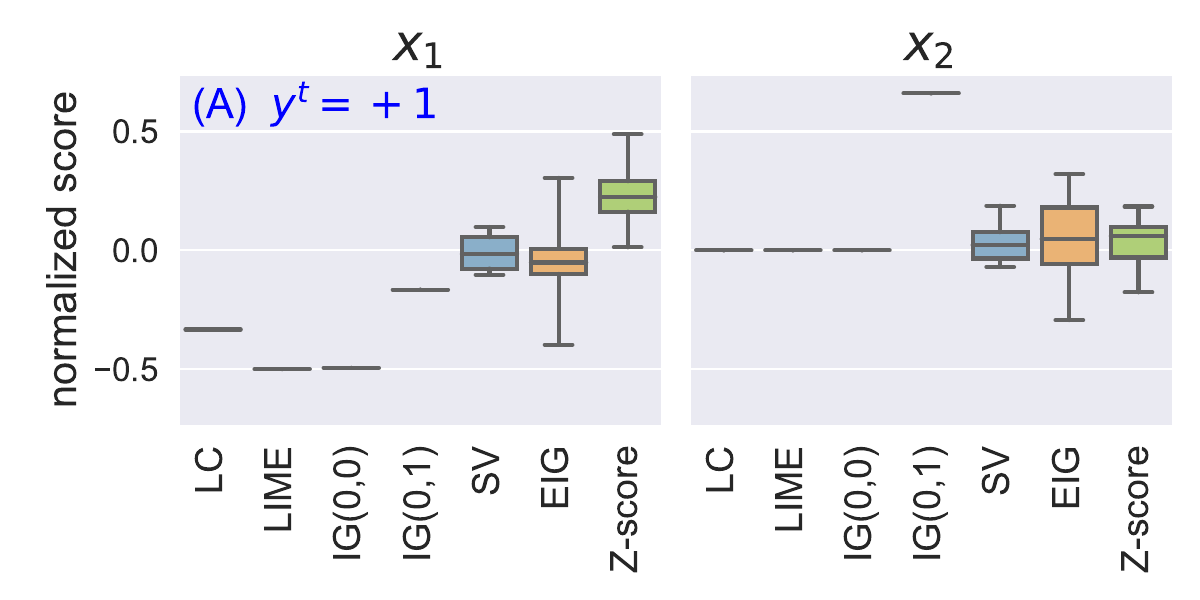}
\hspace{-3mm}	
\includegraphics[trim={0.cm 0.cm 0cm 0.cm},clip,width=7.6cm]{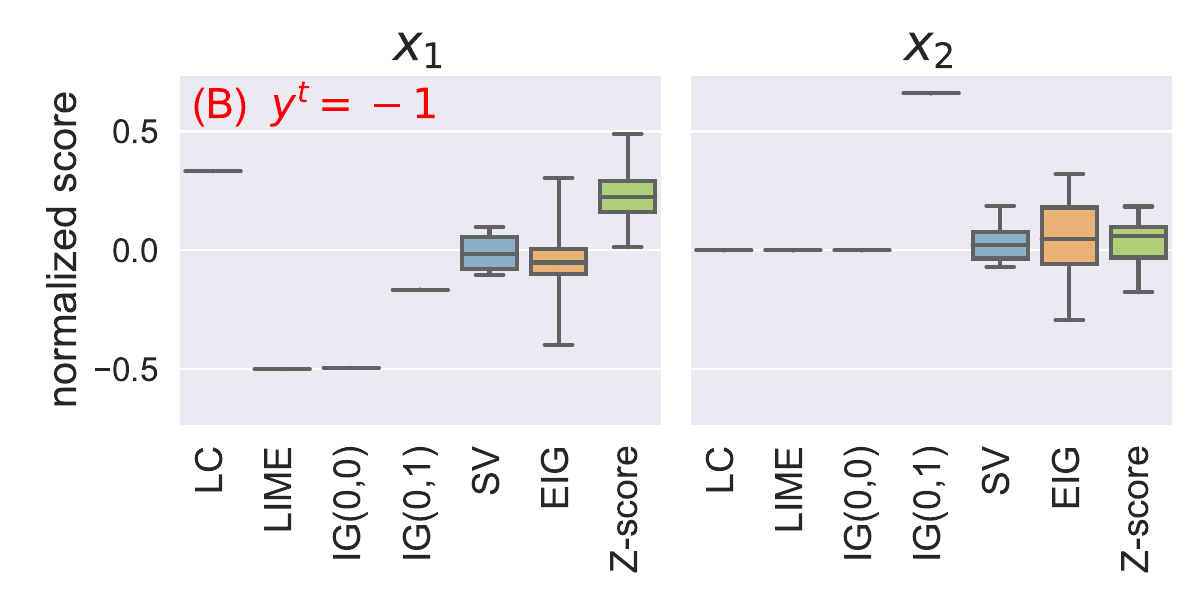}
\caption{Comparison of attribution scores on the 2D Sinusoidal Curve at two test points (A and B in Fig.~\ref{fig:2D_sinusoidal}). The scores were evaluated 10 times over randomly generated datasets. Only the LC scores differ between A and B. }
\label{fig:2D_sinusoidal_scores}
\end{figure}

Figure~\ref{fig:2D_sinusoidal_scores} compares the attribution score, where the mean and the standard deviation over the ten trials is shown. Randomness in LC, LIME and IGs comes only from gradient estimation and is negligible, while EIG, SV and the $Z$-score are directly impacted by the variability of the samples. The attribution scores are normalized by dividing by a scaling factor to make them mutually comparable. The most conspicuous observation from Fig.~\ref{fig:2D_sinusoidal_scores} is that only LC can distinguish the direction of the deviation between A and B; All the other methods produce exactly the same score between A and B due to the \textit{deviation-agnostic} property. As seen from Fig.~\ref{fig:2D_sinusoidal}, $\delta_1$ (the LC score for $x_1$) was negative for point A. This can be easily understood from the definition of LC as `horizontal deviation.' For the given $y^t=1$ value, the location of the point is unusual; it should have been a little more to the left (so it gets the maximum reward of likelihood). Similarly, for point B, $x_1=0.5$ was unusual for $y^t=-1$ and it should have been a little more to the right to match $y^t= -1$. 

We see that SV, EIG, and $Z$-scores have significant variability in contrast to LC, LIME and IGs. As discussed in Sec.~\ref{subsec:IG_and_EIG}, EIG is reduced to IG when the probability mass of $P(\bmx)$ is concentrated at the baseline input of IG. In this case, however, the distribution has been chosen to be a broad uniform distribution, which is at the opposite end of the spectrum from Dirac's delta function. In such a case, EIG generally has a large variability. For the cases where $P(\bmx)$ is close to unimodal, see Sec.~\ref{sec:variability_of_EIG_SV}. 

We also see that the attribution scores of IG significantly vary depending on the choice of the baseline input, which is either $(0,0)$ or $(0,1)$. EIG eliminates the need for arbitrary input by expectation. However, empirical expectation resulted in large error bars. These characteristics of the baseline methods, along with their deviation-agnostic property, make LC the preferred method for the task of anomaly attribution.

\subsection{Direct interpretability of LC}

\paragraph{Boston Housing data}

As an example of a real-world attribution task, we next use Boston Housing data~\cite{belsley2005regression}, a well-known benchmark data for regression. The task is to predict $y$, the median home price (`MEDV') of the districts in Boston, with $\bmx$, the input vector of size $M=13$, such as the percentage of the lower status of the population (`LSTAT') and the average number of rooms (`RM')\footnote{We excluded a variable named `B' from attribution for ethical concerns~\cite{sklearn113_Boston_Housing}.}. The total number of samples is 506. As one might expect, the data is very noisy. Figure~\ref{fig:Boston_2outliers_5variables_scatter} shows scatter plots between $y$ and six selected input variables, where multimodal, clustered structures are observed. We standardized the entire dataset for each variable to be zero mean and unit variance, then we held out $20$\% of the data as $\mathcal{D}_\mathrm{test}$, and trained the random forest (RF)~\cite{ESL2} on the rest. Viewing it as a black-box regression model $f(\bmx)$, we computed the anomaly score with Eq.~\eqref{eq:changeScoreDef}, based on $p(y \mid \bmx)$ estimated with Eqs.~\eqref{eq:p(y|x)_Gaussian}-\eqref{eq:sigma2_for_test_samples}, where $w_n$ is set to a constant on the held-out samples. Note that the training dataset is supposed to be unavailable in the doubly black-box setting. We explicitly use the training data here for comparison purposes.

\begin{figure}[tb]
\centering
\includegraphics[trim={0.cm 0.cm 0cm 0cm},clip,width=15cm]{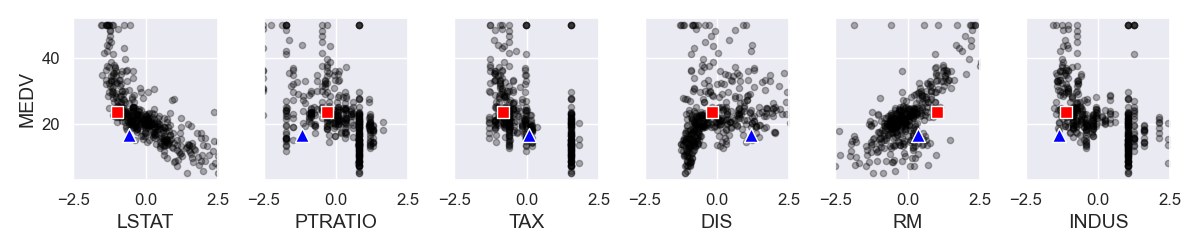}
\vspace{-2mm}
\caption{Boston Housing: Pairwise scatter plot between $y$ (MEDV) and selected input variables. The square and triangle show the detected first and second outliers in the test data. }
\label{fig:Boston_2outliers_5variables_scatter}
\end{figure}

Figure~\ref{fig:anomaly_score_boston_calif_diabetes.pngf}~(a) shows the computed anomaly scores. The first and second outliers are highlighted with the square and triangle symbols, respectively, which have also been shown in Fig.~\ref{fig:Boston_2outliers_5variables_scatter} with the same symbols. Unlike the general expectation from the word `outlier,' those anomalous samples are not necessarily in the low-density region in the scatter plot. One reason for this is that the pairwise scatter plot visualize $p(y^t \mid x^t_i)$ for each $i$ rather than $p(y^t \mid \bmx^t)$, on which the anomaly score was calculated. It is well-known that the marginal distribution makes interpretation tricky when deviations are important~\cite{molnar2019interpretable,kumar2020problems}. We strive to get a more sophisticated explanation specific to a test sample rather than aggregated signals in the marginals. 

\begin{figure}[tb]
\centering
\includegraphics[trim={0.cm 0.0cm 0cm 0cm},clip,width=\textwidth]{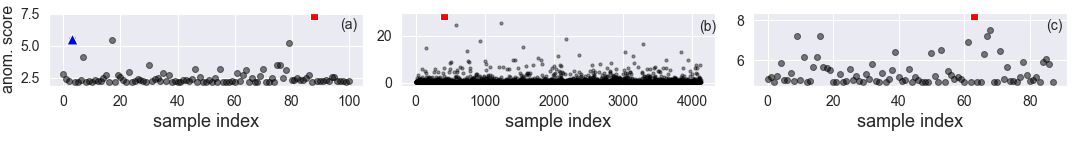}
\vspace{-5mm}
\caption{Anomaly score for (a) Boston Housing, (b) California Housing, and (c) Diabetes datasets. Top outliers are highlighted. }
\label{fig:anomaly_score_boston_calif_diabetes.pngf}
\end{figure}

Figure~\ref{fig:Boston_2outlier_score_comparison} shows attribution scores computed for the first and second outliers. We used $\lambda=0.5,\nu=0.1,\kappa=0.1$ for LC, which were adjusted so the entropy of the absolute score distribution is roughly the same as that of IG. The same regularization parameter was used for LIME. As discussed in Sec.~\ref{sec:EIG}, IG, EIG, and SV are pseudo-local methods in the sense that they require either a global distribution or a baseline input outside the vicinity of the test input, while LC and LIME are purely local attribution framework. To make LC and LIME comparable to the others, we used a relatively large $\eta=1$ for gradient estimation. For IG, the baseline input was equated to the mean of $\bmx$. 
In Fig.~\ref{fig:Boston_2outlier_score_comparison}~(1), all the methods except for the $Z$-score gave similar score distributions, where LSTAT and RM have dominating weights. Close inspection shows, however, that there are interesting differences in their signs. In the gradient-based approach LIME, LSTAT and RM give negative and positive signs, respectively, following the global trend in the scatter plots in Fig.~\ref{fig:Boston_2outliers_5variables_scatter}. As discussed in the previous subsection, LIME does not take account of $y$ of the test point. If a variable has a globally monotonic distribution, LIME tends to provide a similar score, regardless of the $y$ value. In contrast, LC's attribution score is directly interpretable. In this case, LC was negative for RM because a negative shift would give a better fit: ``The number of rooms is a bit too large for the (relatively low) price.'' In addition, LC's attribution score represents an actual shift. We can readily confirm that the red square would fall into the densest region of the RM-MEDV scatter plot with a shift by about 0.5 to the left in Fig.~\ref{fig:Boston_2outliers_5variables_scatter}.

\begin{figure}[bth]
\centering
\includegraphics[trim={0.cm 0.2cm 0cm 0cm},clip,width=7.9cm]{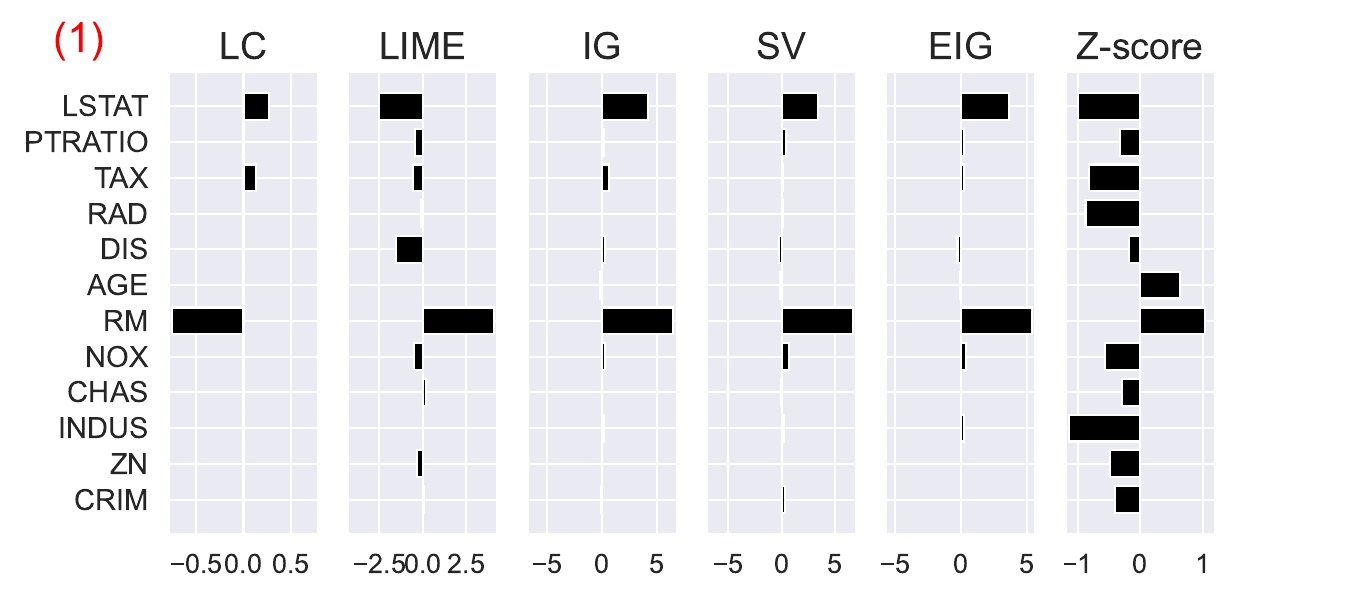}
\hspace{-8mm}
\includegraphics[trim={0.cm 0.2cm 0cm 0cm},clip,width=7.9cm]{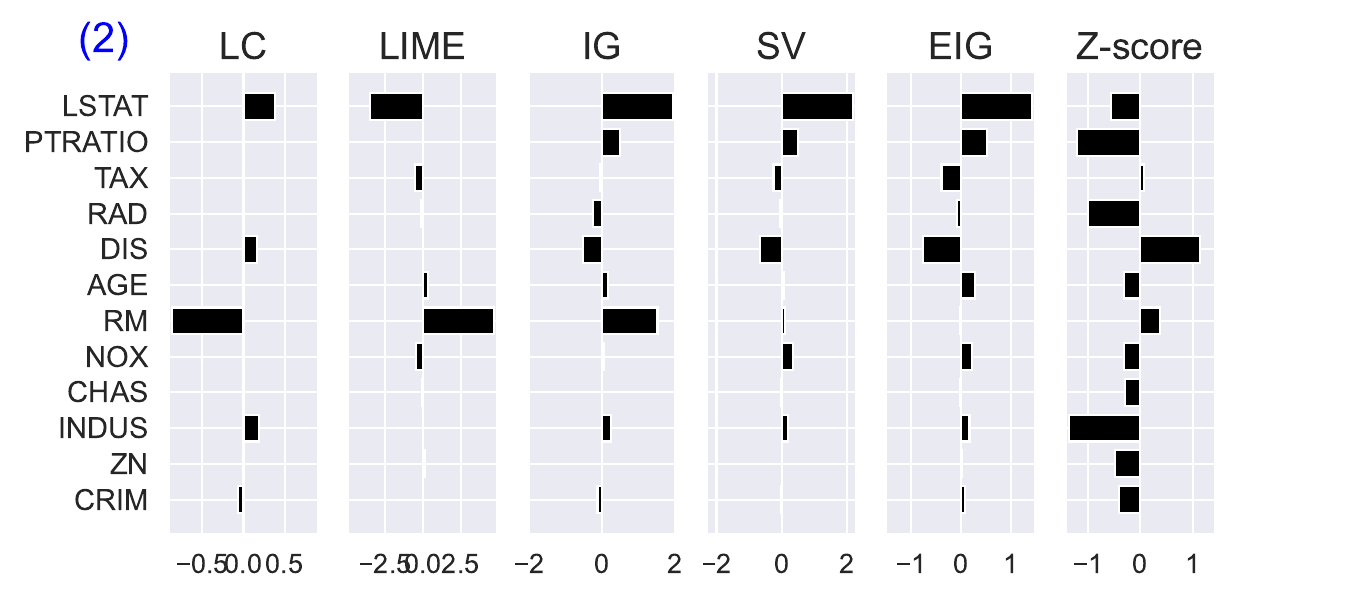}
\caption{Boston Housing: Comparison of the attribution scores for (1) the first and (2) the second outliers highlighted with the square (red) and triangle (blue) symbols, respectively, in Fig.~\ref{fig:Boston_2outliers_5variables_scatter}. }
\label{fig:Boston_2outlier_score_comparison}
\end{figure}

\paragraph{California Housing data}


To confirm general applicability of LC, we next used the California Housing dataset~\cite{pace1997sparse}. This is a relatively large data set of $20\,640$ samples with 9 predictor variables, of which we log-transformed four variables (`AveRooms', `AveBedrms', `Population', `AveOccup'). The task is to predict the median house value of small geographical segments using predictor variables such as the median household income of each segment. We randomly held out 20\% of the samples after standardization and trained gradient boosted trees (GBT)~\cite{friedman2002stochastic} on the rest. Following the same procedure as that for the Boston Housing dataset, we computed anomaly scores for the held-out samples as shown in Fig.~\ref{fig:anomaly_score_boston_calif_diabetes.pngf}~(b). Then we took the one with the highest anomaly score (highlighted with the red rectangle) as the test sample $(\bmx^t,y^t)$. 

Figure~\ref{fig:California_scatter_406_seed823} shows pairwise scatter plots of selected variables against the target variable (`MedHouseVal'). An interesting bi-modal structure is observed in the two variables. As clearly seen from the figure, the test sample is off the main cluster in many variables. From the figure, we see, for example, that the median income of the segment (`MedInc') is a bit too small and the average number of household members (`AveOccup') is a bit too large. One interesting question is whether these observations are consistent with LC's attribution. 

To answer this question, we compare computed attribution scores in Fig.~\ref{fig:California_score_comparison_406_seed823.pdf}. We used $\lambda=0.4,\nu=0.2,\kappa=0.1$ for LC, which were adjusted using the same approach as the Boston case. The same regularization parameter was used for LIME. IG's baseline input was set to be the origin (the population mean of $\bmx$). For EIG, SV, and the $Z$-score, we used empirical approximation using 100 bootstrapped samples from the training data to simulate a `semi-doubly black-box' situation, where only a limited number of test samples are available (see Sec.~\ref{sec:variability_of_EIG_SV} for more detail). The bootstrap approach allows estimating the distribution of attribution scores (see the next subsection for the detail). The $Z$-score confirms that MedInc and AveOccup are significantly smaller and larger than the mean, respectively. However, the attribution scores for the latter are almost zero in all the five methods. This is a clear example where the deviation in $\bmx$ does not necessarily mean being an outlier in regression. Specifically, since AveOccup is almost unrelated with the target variable, as suggested by the distribution in Fig.~\ref{fig:California_scatter_406_seed823}, any shifts in that variable should not decisively explain the anomalousness. 

For MedInc, on the other hand, LC pinpoints that it is the biggest contributor and a large positive attribution score provides a counterfactual explanation: If MedInc were a bit higher, the sample would have looked less unusual. In other words, the median income was a bit too low for the median house value.

\begin{figure}[tb]
\centering
\includegraphics[trim={0.cm 0.cm 0cm 0cm},clip,width=15cm]{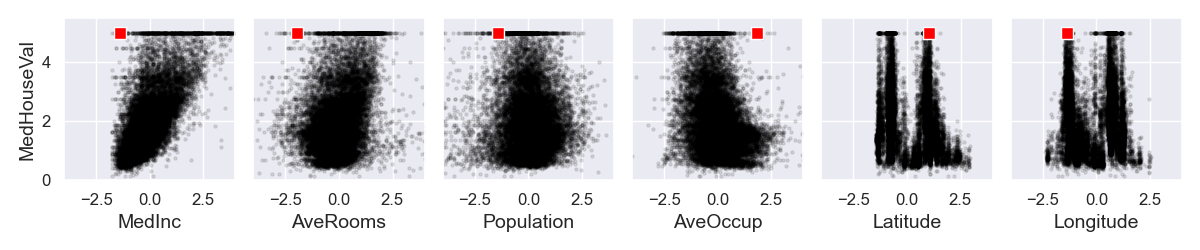}
\vspace{-2mm}
\caption{California Housing: Pairwise scatter plot between $y$ (`MedHouseVal') and selected input variables. The square symbol indicates the top outlier in Fig.~\ref{fig:anomaly_score_boston_calif_diabetes.pngf} (b). }
\label{fig:California_scatter_406_seed823}
\end{figure}

\begin{figure}[tb]
\centering
\includegraphics[trim={0.cm 0.cm 0cm 0cm},clip,width=14cm]{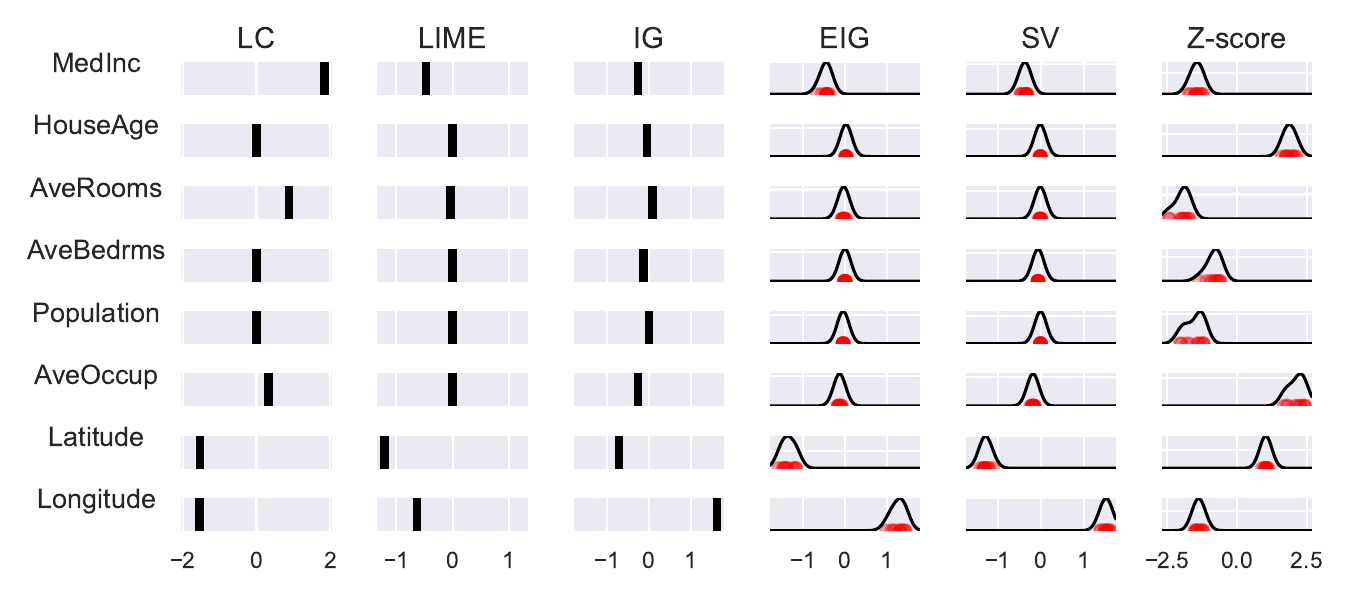}
\vspace{-2mm}
\caption{California Housing: Comparison of attribution scores for the outlier highlighted with the red square in Fig.~\ref{fig:California_scatter_406_seed823}. }
\label{fig:California_score_comparison_406_seed823.pdf}
\end{figure}

\subsection{Variability of EIG and SV} \label{sec:variability_of_EIG_SV}

\paragraph{Boston Housing} 

Let us go back to Figs.~\ref{fig:Boston_2outlier_score_comparison}~(1) and~(2) of Boston Housing. Between these two outliers, LC, LIME, and IG exhibit consistent scores. Interestingly, this is not the case in SV and EIG: IG's large weight on RM vanishes in EIG and SV, resulting in significantly different attributions. Since the baseline input of IG was set to be the mean of $\bmx$, one would expect largely consistent attributions between IG and EIG (and hence, SV, according to Theorem~\ref{th:SV=EIG}). Recall  Fig.~\ref{fig:problem_setting_and_deviation_v_increment}~(b) and the sum rule in Eq.~\eqref{eq:EIG's_efficiency}. In this particular case, IG has a large positive weight on RM because the mean is on the left to the test points, and, as going from the mean to the test points, $f(\bmx)$ value goes up along the globally positive gradient. In (2), the test point is closer to the mean than (1), as indicated by their $Z$-scores; Its RM value is located in the densest region of RM. In such a case, the contribution of the variable tends to be unpredictable because the paths from a test point can be in almost any direction, even from the right, and averaging out those paths can result in either a large or a near-zero value. In general, the attribution of EIG can be counter-intuitive when the test input is close to the population mean of $\bmx$ and the local gradient is steep. 

This ``vanishing weight'' issue originates mainly from the fact that IG and EIG (and thus, SV) explain the \textit{increment} rather than the deviation. We see the test points as outliers  Fig.~\ref{fig:Boston_2outlier_score_comparison}~(2) because of the large deviations in $\bmy$, i.e.\ the vertical shifts. However, the deviation-agnostic methods do not see the test points in that way. This relatively simple example empirically demonstrates the subtle but intricate difficulty in using deviation-agnostic algorithms for anomaly attribution.

\begin{figure}[tb]
\centering
\includegraphics[trim={0.cm 0.3cm 0cm 0cm},clip,width=13cm]{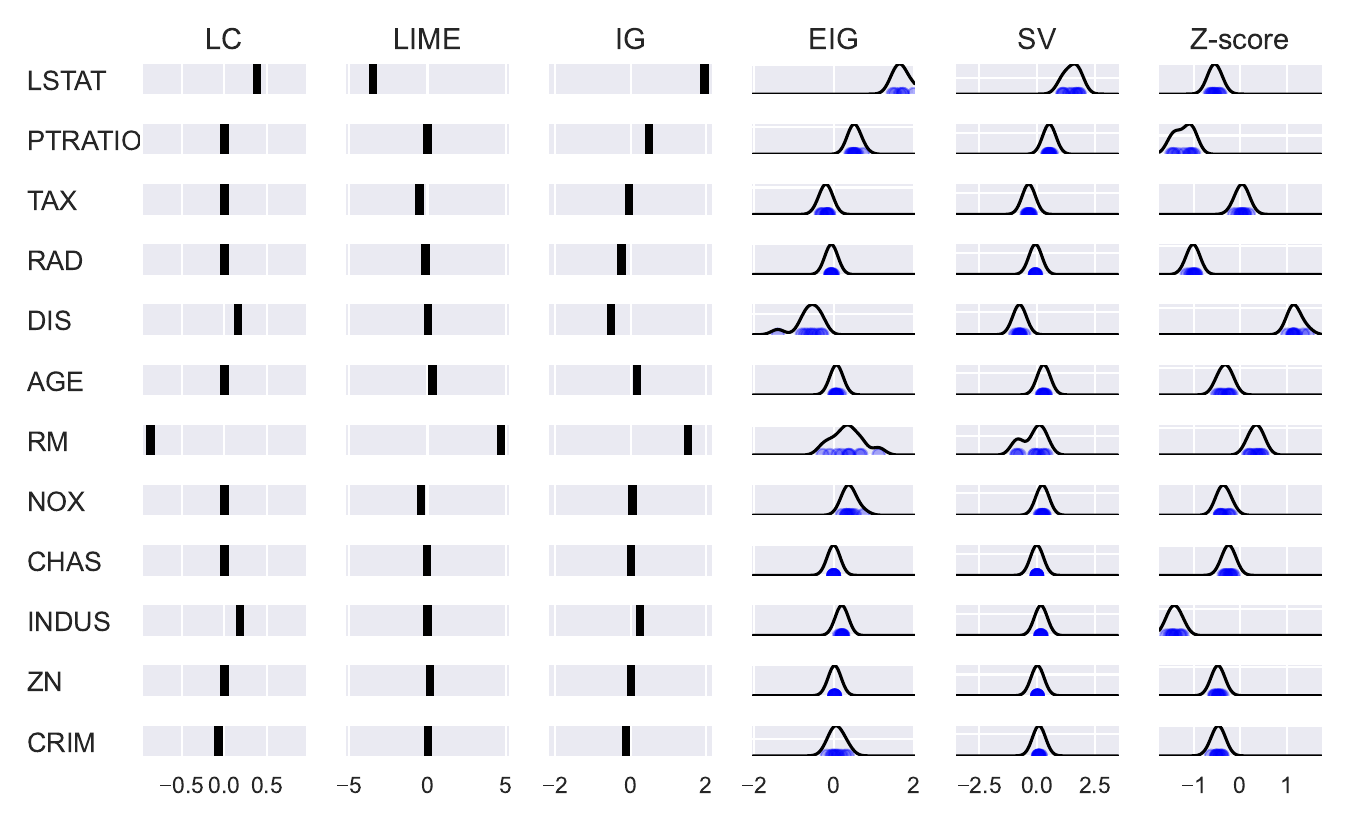}
\vspace{-0.2cm}
\caption{Boston Housing: Comparison of attribution scores with simulated score variability for the outlier highlighted with the blue triangle in Fig.~\ref{fig:Boston_2outliers_5variables_scatter}. }
\label{fig:Boston_score_variability}
\end{figure}

As mentioned earlier, EG, SV, and the $Z$-score are basically inapplicable to anomaly attribution in the doubly black-box setting. If there is some amount of samples, however, one can compute these quantities somehow by taking the empirical average, as we have done above. One question here is whether those methods can produce stable attribution scores under a limited number of samples. To simulate such a situation, we created 10 sets of bootstrapped datasets of $N_{\mathrm{b}}=100$, generated from the training set of the Boston Housing data. Between the two outliers in Fig.~\ref{fig:Boston_2outlier_score_comparison}, we focused on the second one as it has much more diversity in the score. The expectation in EIG and SV was performed using the empirical approximation by the bootstrapped samples. The distribution of the attribution scores are shown in Fig.~\ref{fig:Boston_score_variability} for EIG, SV, and the $Z$-score, where the points denote computed score values of the 10 bootstrap rounds. The curves were estimated Gaussian-based kernel density estimation\footnote{We used the \texttt{KernelDensity} implementation of scikit-learn 1.2.0. }
\begin{align}
    p(s_i) = \frac{1}{N_\mathrm{b}}\sum_{l=1}^{N_\mathrm{B}}  
    \calN(s_i \mid s_i^{(l)}, \eta_\mathrm{b}^2),
\end{align}
where $s_i$ denotes the attribution score of the $i$-th variable,  $s_i^{(l)}$ denotes the computed $s_i$ value on the $l$-th bootstrapped replica (dataset), and $N_\mathrm{B}$ is the number of bootstrapped replicas, which is 10 in this case. Also, the bandwidth $\eta_\mathrm{b}$ was set to 4\% of the range for each variable. As LC, LIME and IG do not need $P(\bmx)$, their scores are the same as those presented in Fig.~\ref{fig:Boston_2outlier_score_comparison}. The same approach was used to draw the curves in Fig.~\ref{fig:California_score_comparison_406_seed823.pdf}. 

From Fig.~\ref{fig:Boston_score_variability}, we first see that there is a systematic similarity between SV's and EIG's scores. This is an empirical confirmation of Theorem~\ref{th:SV=EIG}. As SV is computationally demanding in high-dimensional data due to its combinatorial nature, one can use EIG to approximate SV in practice. Given the vanishing weight issue of EIG, even IG could be used as long as prior knowledge naturally defines the baseline input, although these methods are deviation-agnostic and not the best choice for anomaly attribution. Second, in SV and EIG, the score distributes around zero in many variables. This underscores the need for a sparsity-enforcing mechanism in attribution. In contrast, LC and LIME are not distracted by pseudo-signals thanks to the sparsity in their scores. Third, the variability of the score can be extremely large in some variables. In fact, in RM, the standard deviation of the distribution is even larger than the absolute score itself. This means that RM's attribution score is not at all trustworthy. The large variance issue in RM is another manifestation of the vanishing weight issue discussed in the previous subsection. The above observations remained unchanged when we increase $N_{\mathrm{b}}$ e.g.~to 200.

\subsection{Consistency among attribution methods}

Figures~\ref{fig:California_score_comparison_406_seed823.pdf} and~\ref{fig:Boston_score_variability} indicate that LC and the alternative attribution methods are largely consistent at least in the top attributions. To quantitatively evaluate the consistency among different attribution methods, we computed the following four metrics. The first and second metric is Kendall's $\tau$ and Spearman's $\rho$ computed on two \textit{absolute} attribution score vectors. They are typically called the \textit{rank correlation coefficients} and take a value of 1 if the order of the two absolute scores are the same regardless of their values. The third metric is what we call the sign match ratio, which takes on 1 when all the signs are consistent between corresponding vector elements. When comparing an attribution score vector $\bmu$ against a reference score vector $\bmr$, the sign match ratio is defined as
\begin{align}
    (\mbox{sign match ratio}) &\triangleq 1 - \frac{1}{M}\sum_{i=1}^M\mathbb{I}\left(\sign(r_i)\sign(u_i) = -1\right),
\end{align}
where $\mathbb{I}(\cdot)$ is the indicator function that takes on 1 when the argument  is true, 0 otherwise. We define $\sign(0)=0$ in this case. Note that this favors sparse attribution scores: If $\bmr =\bmzero$, then the score is always 1 regardless of $\bmu$. Finally, the fourth metric is what we call hit25, which gives 1 when the top 25\% of the absolute entries perfectly match between $\bmr$ and $\bmu$, and 0 if none of the top 25\% members of $\bmr$ is included in that of $\bmu$. As hit25 depends on neither the sign nor the rank, it quantifies simply the match of top contributors.

\begin{figure}[tb]
\centering
\includegraphics[trim={0.cm 0.3cm 0cm 0cm},clip,width=15cm]{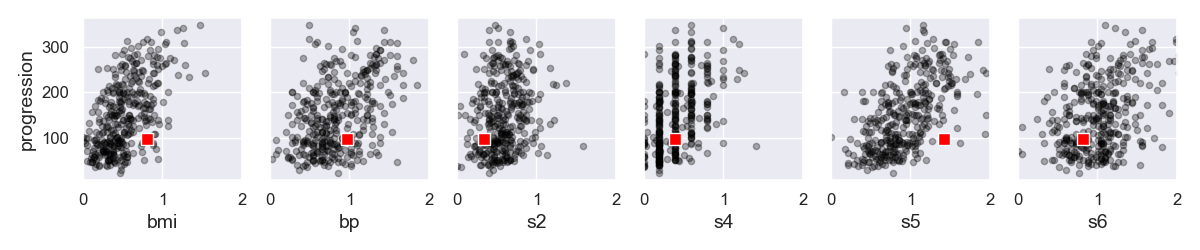}
\vspace{-0.2cm}
\caption{Diabetes: Pairwise scatter plot between $y$ (`progression') and selected input variables. The square highlights the detected top outlier in Fig.~\ref{fig:anomaly_score_boston_calif_diabetes.pngf} (c). }
\label{fig:diabetes_scatter_n63_seed50.pdf}
\end{figure}

\begin{figure}[tb]
\centering
\includegraphics[trim={0.cm 0.3cm 0cm 0cm},clip,width=13cm]{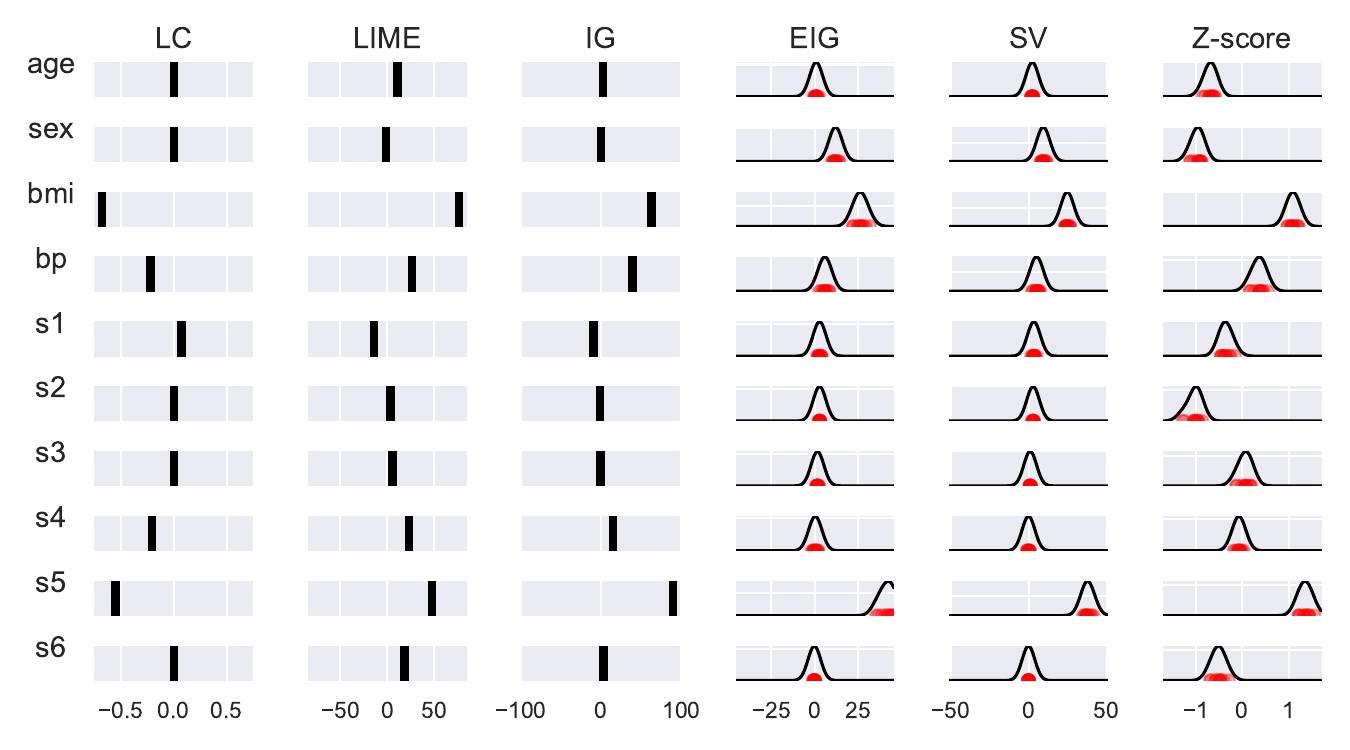}
\vspace{-0.2cm}
\caption{Diabetes: Comparison of attribution scores with simulated score variability for the top outlier highlighted with the square in Fig.~\ref{fig:diabetes_scatter_n63_seed50.pdf}.}
\label{fig:diabetes_score_n63_seed50.pdf}
\end{figure}

\paragraph{Diabetes data} To study the consistency among the attribution methods in a more comprehensive fashion, we added another benchmark dataset called the Diabetes dataset~\cite{efron2004least}, which contains 442 samples of one real-valued target variable (`progression') and $M=10$ predictors including the body-mass index (`bmi'), the average blood pressure (`bp'), and eight other biomarkers. For this dataset, we held-out 20\% of samples after the mini-max scaling with in [0,1] and trained a deep neural network (DNN) on the rest. The resulting model identified has two hidden layers of 32 and 8 neurons with the ReLU (rectified linear unit) activation. The average test $R^2$ score was $0.54$. We thought of the trained DNN as a black-box regression function $y=f(\bmx)$. Following the same procedure as the other datasets, we computed the anomaly score presented in Fig.~\ref{fig:anomaly_score_boston_calif_diabetes.pngf}~(c), where the top outlier is highlighted with the red rectangle. Figure~\ref{fig:diabetes_scatter_n63_seed50.pdf} shows where the outlier point is located in the pairwise scatter plots against the target variable.

For this outlier as the test point $(\bmx^t,y^t)$, we computed attribution scores as shown in Fig.~\ref{fig:diabetes_score_n63_seed50.pdf}. We used the same parameters as those for the California Housing experiment. The baseline input for IG is also at the origin although the origin is no longer the population mean due to the min-max scaling. In this particular test point, LIME, IG, EIG, and SV look quite consistent. LIME simply captures the overall positive trend. IG also captures the positive increment when going from zero (the baseline input in this case) to the test point except for s2 (and s1, which is not shown). As the population mean is close to zero, EIG and SV follow a similar trend.

\paragraph{Comprehensive consistency analysis}

\begin{table}[tb]
\centering
\caption{Result of consistency analysis. The mean and the standard deviation of the metrics are shown in each cell, where 1 represents the highest consistency. }
\label{table:comprehensive_comparision_consistency}
\footnotesize
\begin{tabular}{llllllll}
\hline \hline
                          &            & LC & LIME & IG & EIG & SV & $Z$-score \\
\hline
\multirow{4}{*}{Boston}   & $\tau$     & $1.00\pm 0.00$   &$0.49\pm 0.31$ & $0.61\pm 0.09$ & $0.40 \pm 0.07$ & $0.53\pm 0.16$ & $0.15\pm 0.34$ \\
                          & $\rho$     & $1.00\pm 0.00$   &$0.59\pm 0.32$ & $0.72\pm 0.07$ & $0.49 \pm 0.09$ & $0.61\pm 0.19$ & $0.18\pm 0.39$ \\
                          & sign &    $1.00\pm 0.00$    &$0.68\pm 0.10$ & $0.83\pm 0.08$ & $0.80 \pm 0.09$ & $0.86\pm 0.06$ & $0.74\pm 0.12$ \\
                          & hit25    &$1.00\pm 0.00$    &$0.80\pm 0.18$ & $0.80\pm 0.30$ & $0.60 \pm 0.15$ & $0.67\pm 0.24$ & $0.27\pm 0.28$ \\
\hline
\multirow{4}{*}{California}   & $\tau$     & $1.00\pm 0.00$   &$0.87\pm 0.04$ & $0.75\pm 0.10$ & $0.50 \pm 0.12$ & $0.52\pm 0.13$ & $0.08\pm 0.15$ \\
                          & $\rho$     &$1.00\pm 0.00$    &$0.93\pm 0.03$ & $0.89\pm 0.04$ & $0.65 \pm 0.15$ & $0.67\pm 0.15$ & $0.16\pm 0.21$ \\
                          & sign    & $1.00\pm 0.00$   &$0.48\pm 0.04$ & $0.48\pm 0.04$ & $0.62 \pm 0.13$ & $0.66\pm 0.15$ & $0.66\pm 0.05$ \\
                          & hit25    &$1.00\pm 0.00$    &$1.00\pm 0.00$ & $0.90\pm 0.22$ & $0.80 \pm 0.27$ & $0.80\pm 0.27$ & $0.30\pm 0.44$ \\
\hline
\multirow{4}{*}{Diabetes} & $\tau$     &$1.00\pm 0.00$    &$0.67\pm 0.08$ & $0.53\pm 0.12$ & $0.52 \pm 0.21$ & $0.54\pm 0.19$ & $-0.04\pm 0.19$ \\
                          & $\rho$     & $1.00\pm 0.00$   &$0.78\pm 0.07$ & $0.67\pm 0.13$ & $0.66 \pm 0.17$ & $0.66\pm 0.19$ & $-0.09\pm 0.27$ \\
                          & sign    &$1.00\pm 0.00$   &$0.85\pm 0.22$ & $0.68\pm 0.11$ & $0.65 \pm 0.14$ & $0.65\pm 0.14$ & $0.65\pm 0.16$ \\
                          & hit25    & $1.00\pm 0.00$   &$0.90\pm 0.22$ & $0.80\pm 0.27$ & $0.80 \pm 0.27$ & $0.80\pm 0.27$ & $0.20\pm 0.27$\\
\hline
\end{tabular}
\end{table}

On the other hand, LC looks like the LIME attribution scores with the opposite sign in Fig.~\ref{fig:diabetes_score_n63_seed50.pdf}, which was also the case in Fig.~\ref{fig:Boston_score_variability}. One interesting question here is whether or not there is a systematic correspondence of LC's attribution to the other methods, especially when we ignore the signs. Among the four metrics, $\tau,\rho$, and hit25 are about the absolute scores. In the three data sets, we picked top five outliers and computed the attribution scores for them. For each, we computed the four metrics with the reference $\bmr$ being the LC's score. Those metrics are designed to capture the consistency between the top contributors, disregarding minor contributors. LC is the only method that has both a built-in sparsity-enforcing mechanism and the deviation-sensitive property, making it an appropriate choice for the reference. 

The result is summarized in Table~\ref{table:comprehensive_comparision_consistency}. As expected, hit25 has generally high scores, apart from the $Z$-score. This suggests that those attribution methods are a useful tool to select important features. Even in the other metrics including the sign match ratio, they produce reasonably consistent attributions in some cases. However, some 20-30\% of cases are still not necessarily consistent, which is a natural consequence that LC is deviation-sensitive but the others are not.

\subsection{Real-world business use-case: building energy management}
\label{subsec:bulding experiments}

Finally, we provide an example about how LC can make a difference in a real business use-case. Collaborating with IBM IoT Business Unit, we obtained energy consumption data for an office building in India. The total wattage $y$ is predicted by a black-box commercial prediction tool as a function of weather-related (temperature, humidity, etc.) and time-related variables (time of day, day of week, month, etc.). There are two intended usages of the predictive model. One is near future prediction with short time windows for optimizing HVAC (heating, ventilating, and air conditioning) system control. The other is \textit{retrospective} analysis over the last few months for the purpose of planning long-term improvement of the building facility and its management policies. In the retrospective analysis, it is critical to get clear explanation on unusual events.

At the beginning of the project, we interviewed 10 professionals on what kind of model explainability would be most useful for them. Their top priority capabilities were uncertainty quantification in forecasting and anomaly diagnosis in retrospective analysis. Our choices in the current research reflect these business requirements.

\begin{figure}[tbh]
\begin{center}
\includegraphics[trim={0.cm 0cm 0cm 0cm},clip,width=8cm]{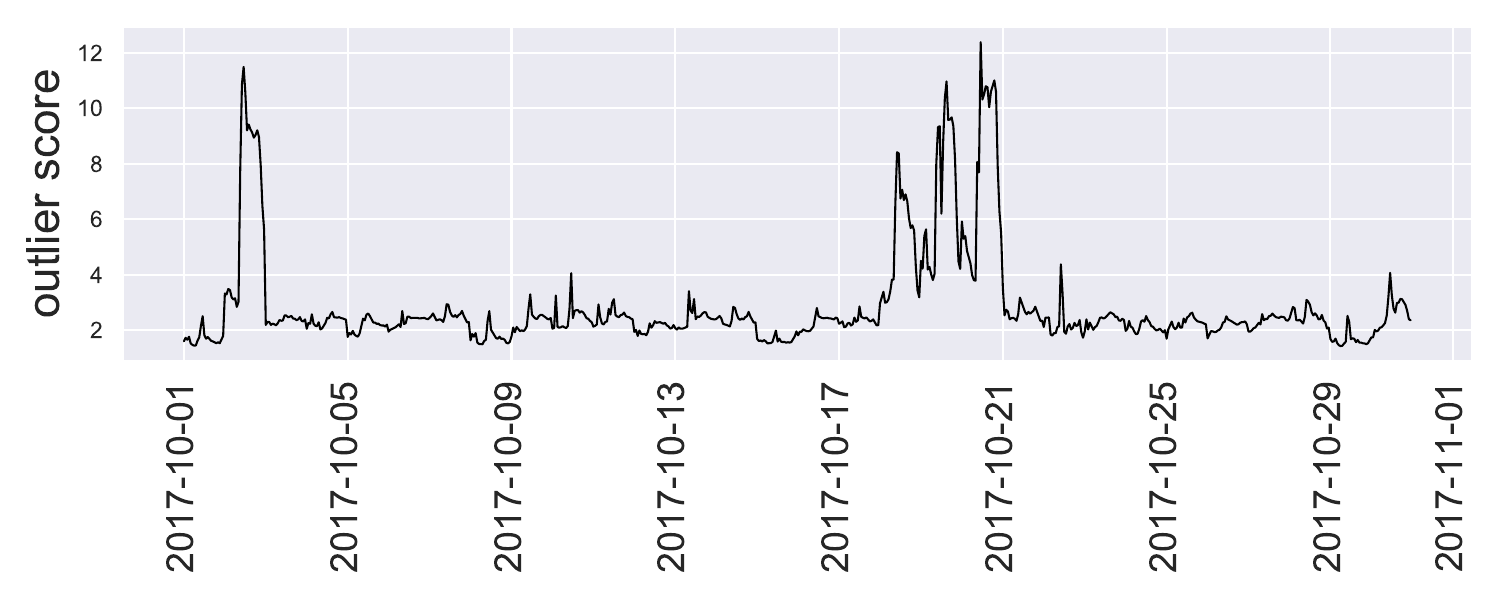}
\end{center}
\vspace{-0.4cm}
\caption{Building Energy: Outlier score computed with Eq.~\eqref{eq:changeScoreDef} for the test data. }
\label{fig:Building_anomalyScore.pdf}
\vspace{-0.4cm}
\end{figure}

\paragraph{Anomaly detection}

We obtained a one month worth of test data with $M=12$ input variables recorded hourly. We first validated the assumed Gaussian observation model by performing kernel density estimation for the probability density function (p.d.f.)~of $y$. As shown in Fig.~\ref{fig:Building_kde_comparison_GBTree_29_highQ.pdf}, the p.d.f.~of $y$ itself is double-peaked, corresponding to different consumption patterns between night and day. On the other hand, the p.d.f.~of $y^t - f(\bmx^t)$ in the right panel is single-peaked, which confirms the validity of the Gaussian-based model in Eq.~\eqref{eq:obs_model}.

\begin{figure}[bt]
\begin{center}
\includegraphics[trim={0.5cm 0cm 1cm 0cm},clip,width=11cm]{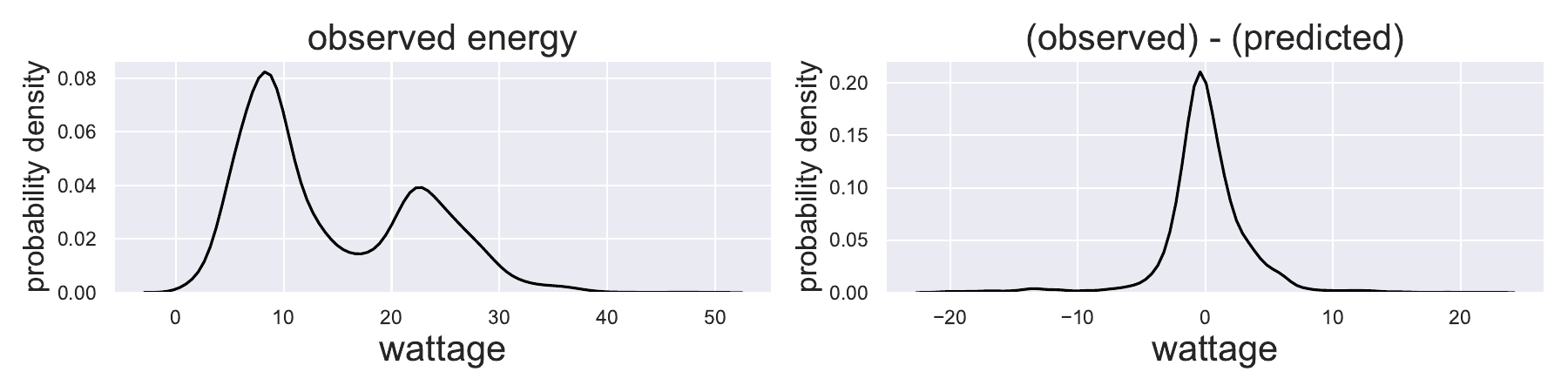}
\end{center}\vspace{-0.3cm}
\caption{Building Energy: Estimated probability densities. Although raw wattage value $y$ (left) is far from Gaussian, the deviation $y-f(\bmx)$ is well approximated by the Gaussian. }
\label{fig:Building_kde_comparison_GBTree_29_highQ.pdf}
\end{figure}

Next, we computed the anomaly score by Eq.~\eqref{eq:changeScoreDef} for each sample $(y^t,\bm{x}^t)$ under the Gaussian observation model. The variance $\sigma^2_t$ was computed using Eq.~\eqref{eq:sigma2_for_test_samples} for each $t$ by leaving $(y^t,\bm{x}^t)$ out from the dataset. The computed anomaly score is shown in Fig.~\ref{fig:Building_anomalyScore.pdf}. We see that there are two periods showing conspicuous anomalies, namely, October 2 and 18-20. An important business question was who or what may be responsible for these anomalies.

\begin{figure*}[t]
\begin{center}
\includegraphics[trim={0.2mm 0.5cm 2.cm 0cm},clip,width=\textwidth]{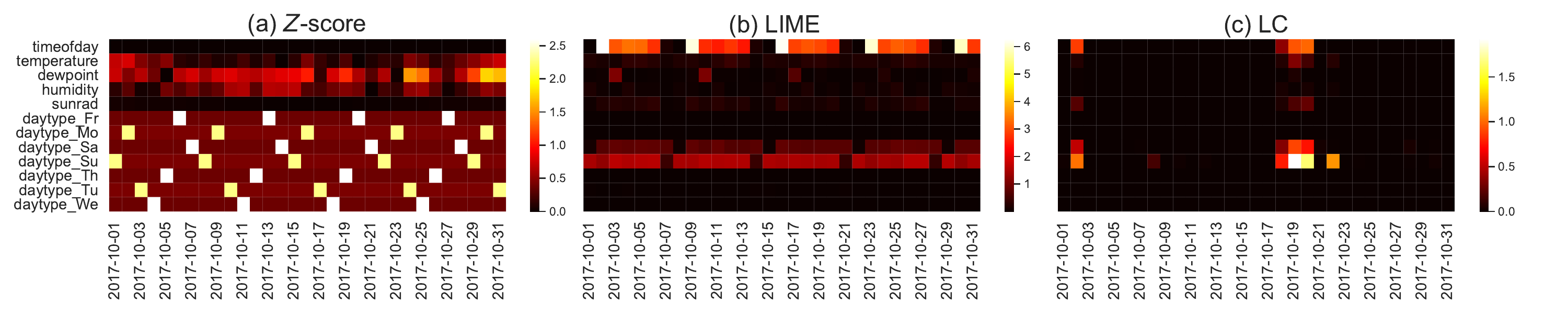}
\end{center}
\caption{Building Energy: Comparison of the explainability scores computed for the test data. }
\label{fig:Building_heatmap_GBTree_29_highQ.pdf}
\vspace{-0.3cm}
\end{figure*}

\paragraph{Anomaly attribution}

To obtain insights regarding the detected anomalies, we computed the LC score as shown in Fig.~\ref{fig:Building_heatmap_GBTree_29_highQ.pdf}, where we computed $\bm{\delta}$ each day with $N_\mathrm{test}=24$ in Eq.~\eqref{eq:LC_Gaussian_elastic_net}, and visualized $\|\bm{\delta}\|_2^2$. We also compared two alternative methods that were applicable: For the $Z$-score, we visualized the daily mean of the absolute values. For LIME, we computed regression coefficients for every sample, and visualized the $\ell_2$ norm of their daily mean. We used $(\nu, \lambda)=(0.1,0.5)$, which was determined by the level of sparsity and scale preferred by the domain experts. IG, EIG, and SV need either a baseline input or training data, and were inapplicable to this real-world setting. 

As shown in the plot, the LC score clearly highlights a few variables whenever the outlier score is exceptionally high in Fig.~\ref{fig:Building_anomalyScore.pdf}, while the $Z$-score and LIME do not provide much information beyond the trivial weekly patterns. The pattern of LIME was very stable over $0 <\nu \leq 1$, showing empirical evidence of the deviation-agnostic property. On the other hand, the $Z$-score sensitively captures the variability in the weather-related variables, but it fails to explain the deviations in Fig.~\ref{fig:Building_anomalyScore.pdf}. This is understandable because the $Z$-score does not reflect the relationship between $y$ and $\bm{x}$. The artifact seen in the ``daytype'' variables is due to the one-hot encoding of the day of week.

With LC, the strongest signal is observed around October 19 (Thursday) in Fig.~\ref{fig:Building_anomalyScore.pdf}. The variables highlighted are `timeofday', `daytype\_Sa', and `daytype\_Su', implying that those days had an unusual daily wattage pattern for a weekday and looked more like weekend days. Interestingly, it turned out that the 19th was a national holiday in India and many workers were off on and around that date. The other anomalous period on October 2 was also a national holiday. Thus we conclude that the anomaly is most likely not due to any faulty building facility, but due to the model limitation caused by the lack of full calendar information. Though simple, such pointed insights made possible by our method were highly appreciated by the professionals.\footnote{LC has been productized as IBM's software offering.}

\section{Conclusions}

We have proposed a new framework for model-agnostic anomaly attribution in the doubly black-box regression setting. We mathematically proved that integrated gradient IG), local linear surrogate modeling (LIME), and Shapley values (SV) are inherently deviation-agnostic and thus, cannot be a viable solution for anomaly attribution. We have clarified a mathematical structure leading to the deviation-agnostic property using a power expansion technique. Unlike these methods, the proposed likelihood compensation approach is built upon the maximum likelihood principle, and is capable of capturing specific characteristics of anomalies observed. We conducted a comprehensive empirical study using benchmark datasets to verify our mathematical characterization such as an equivalence between SV and IG. We also validated the proposed method on a real-world use-case of building energy management where---based on expert feedback received---the proposed LC method offers significant practical advantages over existing methods.



\vskip 0.2in
\bibliography{ide_et_al}
\bibliographystyle{apalike}

\appendix
\renewcommand{\thesection}{\Alph{section}}
\renewcommand{\theequation}{\Alph{section}.\arabic{equation}}
\renewcommand{\thefigure}{\Alph{section}.\arabic{figure}}
\renewcommand{\thetable}{\Alph{section}.\arabic{table}}

\setcounter{equation}{0}
\setcounter{section}{0}
\setcounter{figure}{0}
\setcounter{table}{0}

\section{Efficiency of Shapley value}
\label{appendix:SV-efficiency}

In this section, we prove the ``efficiency'' of SV in Eq.~\eqref{eq:SV_efficiency}:
\begin{gather*}
    \sum_{i=1}^M \mathrm{SV}_i(\bmx^t) = f(\bmx^t) - \langle f \rangle.
\end{gather*}
\begin{proof}
In
\begin{align}
    \sum_{i=1}^M \mathrm{SV}_i(\bmx^t) 
    = \frac{1}{M} \sum_{i=1}^M
\sum_{k=0}^{M-1} 
\binom{M-1}{k}^{-1}
\sum_{\mathcal{S}_i: |\calS_i|=k} \left[
\langle f \mid x_i^t, \bm{x}_{\mathcal{S}_i}^t \rangle
- \langle f \mid \bm{x}_{\mathcal{S}_i}^t \rangle
\right].
\end{align}
    let us define
    \begin{align}
            I^M_k &\triangleq \frac{1}{M}\sum_{i=1}^M\binom{M-1}{k}^{-1}
     \sum_{\calS_i: |\calS_i|=k}   \langle f\mid x_i^t, \bmx_{\calS_i}^t\rangle,
    \\
    J^M_k &\triangleq \frac{1}{M}\sum_{i=1}^M\binom{M-1}{k+1}^{-1}
    \sum_{{\calS}_i: |{\calS}_i|=k+1}  \langle f\mid \bmx^t_{{\calS}_i}\rangle
    \end{align}
for $k=0, \ldots,M-2$. Equation~\eqref{eq:SV_efficiency} holds if $I^M_k - J^M_k=0$. In $I^M_k$, for a given $k$, the number of distinct sets $\{i\}\cup \calS_i$ is $\binom{M}{k+1}$, and the summations in $I^M_k$ runs over $M\times \binom{M-1}{k}$ terms in total. Hence, each unique set appears
\begin{align}
    M\times \binom{M-1}{k} \times \frac{1}{\binom{M}{k+1}} = k+1
\end{align}
times. Following the same argument, we see that each unique set in $J^M_k$ appears $M-k-1$. 

Let $\calS(k+1)$ be a set of $k+1$ variable indices to represent either $\{i\}\cup \calS_i$ with $|\calS_i|=k$ or $\calS_i$ with $|\calS_i|=k+1$. For each member in $\calS(k+1)$, $I^M_k$ gives a prefactor
\begin{align}
    (k+1) \times  \frac{1}{M}\binom{M-1}{k}^{-1} = \binom{M}{k+1}^{-1},
\end{align}
and $J^M_k$ gives a prefactor
\begin{align}
    (M -  k - 1) \times  \frac{1}{M}\binom{M-1}{k+1}^{-1} = \binom{M}{k+1}^{-1},
\end{align}
which are the same. Hence, we conclude $I^M_k - J^M_k=0$.
\end{proof}

\end{document}

%% file: ide_preamble.tex
\usepackage{capt-of}
\usepackage{algorithm}
\usepackage{algpseudocode}
\usepackage{graphicx}
\usepackage{amssymb}
\usepackage{amsmath}
\usepackage{amsthm} 
\usepackage{color}
\usepackage{bm}

\newtheorem{theorem}{Theorem}

\newtheorem{definition}{Definition}
\usepackage{tcolorbox}
\usepackage{url}
\usepackage{MnSymbol}
\usepackage{stmaryrd}

\DeclareMathOperator{\sign}{sign}

\newcommand{\bme}{\bm{e}}

\newcommand{\bmh}{\bm{h}}

\newcommand{\bmr}{\bm{r}}

\newcommand{\bmu}{\bm{u}}

\newcommand{\bmx}{\bm{x}}
\newcommand{\bmy}{\bm{y}}

\newcommand{\rmd}{\mathrm{d}}

\newcommand{\bmzero}{\bm{0}}

\newcommand{\bmbeta}{\bm{\beta}}

\newcommand{\bmdelta}{\bm{\delta}}

\newcommand{\bmtheta}{\bm{\theta}}

\newcommand{\sfI}{\mathsf{I}}

\newcommand{\calD}{\mathcal{D}}

\newcommand{\calN}{\mathcal{N}}

\newcommand{\calS}{\mathcal{S}}